\documentclass[numbers]{article}

\usepackage[preprint]{neurips_2025}

\usepackage[utf8]{inputenc} 
\usepackage[T1]{fontenc}    

\usepackage{hyperref}       
\usepackage{url}            
\usepackage{booktabs}       
\usepackage{amsfonts}       
\usepackage{nicefrac}       
\usepackage{microtype}      
\usepackage{xcolor}         
\usepackage{graphicx}       
\usepackage{subfigure}      

\usepackage{multirow}

\usepackage[linesnumbered,ruled,vlined]{algorithm2e}

\usepackage{amsmath}
\usepackage{amssymb}
\usepackage{mathtools}
\usepackage{amsthm}

\usepackage[capitalize,noabbrev]{cleveref}

\theoremstyle{plain}
\newtheorem{theorem}{Theorem}[section]
\newtheorem{proposition}[theorem]{Proposition}

\theoremstyle{definition}

\theoremstyle{remark}
\newtheorem{remark}[theorem]{Remark}

\usepackage[textsize=tiny]{todonotes}
\title{Random Forest Autoencoders for Guided Representation Learning}

\author{%
    Adrien Aumon\textsuperscript{\textnormal{1}} \\
  \texttt{adrien.aumon@umontreal.ca} \\
  \And
  Shuang Ni\textsuperscript{\textnormal{1}} \\
  \texttt{shuang.ni@mila.quebec} \\
  \AND
  Myriam Lizotte\textsuperscript{\textnormal{1}} \\
  \texttt{myriam.lizotte@mila.quebec} \\
  \And
  Guy Wolf\textsuperscript{\textnormal{1}} \\
  \texttt{guy.wolf@umontreal.ca} \\
  \And
  Kevin R. Moon\textsuperscript{\textnormal{2}} \\
  \texttt{kevin.moon@usu.edu} \\
    \And
  Jake S. Rhodes\textsuperscript{\textnormal{3}\dag} \\
  \texttt{rhodes@stat.byu.edu} \\
  \\
  \textsuperscript{1}Université de Montréal; Mila -- Quebec AI Institute\\
  \textsuperscript{2}Utah State University\\
  \textsuperscript{3}Brigham Young University\\
  \textsuperscript{\dag}Corresponding author
}

\begin{document}

\maketitle

\begin{abstract}
Extensive research has produced robust methods for unsupervised data visualization. Yet supervised visualization—where expert labels guide representations—remains underexplored, as most supervised approaches prioritize classification over visualization. Recently, RF-PHATE, a diffusion-based manifold learning method leveraging random forests and information geometry, marked significant progress in supervised visualization. However, its lack of an explicit mapping function limits scalability and its application to unseen data, posing challenges for large datasets and label-scarce scenarios. To overcome these limitations, we introduce Random Forest Autoencoders (RF-AE), a neural network-based framework for out-of-sample kernel extension that combines the flexibility of autoencoders with the supervised learning strengths of random forests and the geometry captured by RF-PHATE. RF-AE enables efficient out-of-sample supervised visualization and outperforms existing methods, including RF-PHATE's standard kernel extension, in both accuracy and interpretability. Additionally, RF-AE is robust to the choice of hyperparameters and generalizes to any kernel-based dimensionality reduction method.
\end{abstract}

\section{Introduction}

Manifold learning-based visualization methods, such as $t$-SNE~\cite{vanDerMaaten2008tsne}, UMAP~\cite{McInnes2018umap}, and PHATE~\cite{Moon2019phate}, are essential for exploring high-dimensional data by revealing patterns, clusters, and outliers through low-dimensional embeddings. While these methods excel at uncovering dominant data structures, they often fail to capture task-specific insights when auxiliary labels or metadata are available. Supervised approaches like RF-PHATE~\cite{rhodes2024gaining} bridge this gap by integrating label information into the kernel function through Random Forest-derived proximities~\cite{rhodes2023geometry}, generating representations that align with domain-specific objectives without introducing the exaggerated separations or distortions seen in class-conditional methods~\cite{hajderanj2021impactsupman}. Specifically, RF-PHATE has provided critical insights in biology, such as identifying multiple sclerosis subtypes, demonstrating antioxidant effects in lung cancer cells, and aligning COVID-19 antibody profiles with patient outcomes~\cite{rhodes2024gaining}.

However, most supervised and unsupervised manifold learning algorithms generate fixed coordinates within a latent space but lack a mechanism to accommodate new observations. Therefore, to embed previously unseen data, the algorithm must rerun with the new data as part of the training set. One well-known solution to this lack of out-of-sample (OOS) support is the Nyström extension~\cite{bengio2003out} and its variants, such as linear reconstruction~\cite{roweis2000lle} or geometric harmonics~\cite{coifman2006geometric}. These methods derive an empirical mapping function for new points defined as the linear combination of precomputed training embeddings, weighted with a kernel function expressing the similarities between the new point and the points in the training set. However, these methods are restricted to linear kernel mappings and are based on an unconstrained least-squares minimization problem~\cite{bengio2003out}, which makes them sensitive to the quality of the training set and prone to failing in accurately recovering the true manifold structure~\cite{Rudi2015less, MendozaQuispe2016extreme}. As an alternative extension method, recent neural network-based approaches, such as Geometry-Regularized Autoencoders (GRAE)~\cite{duque2020grae,duque2022geometry}, offer a promising solution for extending embeddings to OOS data points.

In this study, we present Random Forest Autoencoders (RF-AE), an autoencoder (AE) architecture that addresses OOS extension, while taking inspiration from the principles of GRAE~\cite{duque2022geometry}, which uses a manifold embedding to regularize the bottleneck layer. Instead of reconstructing the original input vectors, RF-AE incorporates supervised information by reconstructing Random Forest- Geometry- and Accuracy-Preserving (RF-GAP) proximities~\cite{rhodes2023geometry}, enabling us to learn an embedding function extendable to new data in a supervised context. Furthermore, we introduce a proximity-based prototype selection approach to reduce time and space complexities during training and inference without compromising embedding quality. Notably, our approach extends OOS examples without relying on label information for new points, which is helpful in situations where label information is expensive or scarce.

We empirically show that RF-AE outperforms existing approaches in embedding new data while preserving the local and global structure of the important features for the underlying classification task. RF-AE improves the adaptability and scalability of the manifold learning process, allowing for seamless integration of new data points while maintaining the desirable traits of established embedding methods.

\section{Related work}

\subsection{Parametric embedding through multi-task autoencoders}\label{subsec:multitask_autoencoders}
Given a high-dimensional training dataset \( X = \{\mathbf{x}_i \in \mathbb{R}^D \mid i = 1, \hdots, N\} \)—where \( X \) can represent tabular data, images, or other modalities—a manifold learning algorithm can be extended to test data by training a neural network, typically a multi-layer perceptron (MLP) to regress onto precomputed non-parametric embeddings $\mathbf{z}_i^G$, or by means of a cost function underlying the manifold learning algorithm, as in parametric  $t$-SNE~\cite{maaten2009parametric_tsne} and parametric UMAP~\cite{sainburg2021parametric_umap}. However, solely training an MLP for this supervised task often leads to an under-constrained problem, resulting in solutions that memorize the data but fail to capture meaningful patterns or generalize effectively~\cite{zhang2016understanding, arpit2017closer}. Beyond implicit regularization techniques such as early stopping~\cite{bourlard1989early}, dropout~\cite{Wager2013dropout}, or layer-wise pre-training~\cite{Bengio2006pretraining}, multi-task learning~\cite{Caruana1997multitask} has been shown to improve generalization. Early studies demonstrated that jointly learning tasks reduces the number of required samples~\cite{baxter1995learning, baxter2000inductive_bias}, while later work introduced trace-norm regularization on the weights of a linear, single hidden-layer neural network for a set of tasks~\cite{Maurer2006multitask, pontil2013multitask}. Motivated by Le et al.~\cite{Le2018supervised_autoencoders}, who empirically showed that training neural networks to predict both target embeddings and inputs (reconstruction) improves generalization compared to encoder-only architectures, we focus on multi-task learning-based regularization in the context of regularized autoencoders.

AE networks are built of two parts: an encoder function \( f(\mathbf{x}_i) = \mathbf{z}_i \in \mathbb{R}^d \) (\( d \ll D \)), which compresses the input data into a latent representation via a bottleneck layer~\cite{theis2017lossy}, and a decoder function \( g(\mathbf{z}_i) = \hat{\mathbf{x}}_i \), which maps the low-dimensional embedding back to the original input space. 
To ensure that \(   (g\circ f)(\mathbf{x}_i) = \hat{\mathbf{x}}_i
\approx 
\mathbf{x}_i\),
AEs minimize the average reconstruction loss
$L(f, g) = \frac{1}{N} \sum_{i=1}^N L_{recon}(\mathbf{x}_i, \hat{\mathbf{x}}_i)$
where \( L_{recon}(\cdot, \cdot) \) is typically defined as the squared Euclidean distance.
AEs learn compact data representations that relate meaningfully to the input data. However, standard AEs often fail to capture the intrinsic geometry of the data and do not produce interpretable embeddings~\cite{duque2022geometry}. To address this, regularization techniques have been proposed to guide the optimization process toward better local minima in the latent space.

Structural Deep Network Embedding (SDNE)~\cite{Wang2016sdne} preserves both first- and second-order graph neighborhoods for graph-structured data by combining adjacency vector reconstruction with Laplacian Eigenmaps~\cite{belkin2001laplacian} regularization. Local and Global Graph Embedding Autoencoders~\cite{wang2020lgae} enforce two constraints on the embedding layer: a local constraint to cluster \( k \)-nearest neighbors and a global constraint to align data points with their class centers. VAE-SNE~\cite{graving2020vae-sne} integrates parametric $t$-SNE with variational AEs, enhancing global structure preservation while retaining $t$-SNE's strength in preserving local structure. GRAE~\cite{duque2022geometry} explicitly impose geometric consistency between the latent space and precomputed manifold embeddings.

Other approaches focus on regularizing the decoder. Inspired by Denoising Autoencoders~\cite{vincent2008extracting}, Generalized Autoencoders~\cite{Wang2014gae} minimize the weighted mean squared error between the reconstruction \( \hat{\mathbf{x}} \) and the \( k \)-nearest neighbors of the input \( \mathbf{x} \), where weights reflect the predefined similarities between \( \mathbf{x} \) and its neighbors. Centroid Encoders (CE)~\cite{Ghosh2022centroid} minimize within-class reconstruction variance to ensure that same-class samples are mapped close to their respective centroids. Self-Supervised Network Projection (SSNP)~\cite{espadoto2021ssnp} incorporates neighborhood information by jointly optimizing reconstruction and classification at the output layer, using existing labels or pseudo-labels generated through clustering. Neighborhood Reconstructing Autoencoders~\cite{LEE2021nrae} extend reconstruction tasks by incorporating the neighbors of \( \mathbf{x} \) alongside \( \mathbf{x} \) itself, using a local quadratic approximation of the decoder at \( f(\mathbf{x}) = \mathbf{z} \) to better capture the local geometry of the decoded manifold. Similarly, Geometric Autoencoders~\cite{Nazari2023geometric} introduce a regularization term in the reconstruction loss, leveraging the generalized Jacobian determinant computed at \( f(\mathbf{x}) = \mathbf{z} \) to mitigate local contractions and distortions in latent representations.

\subsection{Kernel methods for OOS extension}\label{subsec:kernel_methods}
Let $k(\cdot, \cdot)$ be a data-dependent symmetric positive definite kernel function $(\mathbf{x}, \mathbf{x}')\mapsto k(\mathbf{x}, \mathbf{x}')\geq 0$. For simplicity, we consider normalized kernel functions that satisfy the sum-to-one property $\sum_{i=1}^N k(\mathbf{x}, \mathbf{x}_i)=1$. Kernel methods for OOS extensions seek an embedding function $\mathbf{k}\mapsto f(\mathbf{k})=\mathbf{z}\in \mathbb{R}^d$ where the input
$
    \mathbf{k} =  \mathbf{k}_{\mathbf{x}} = \begin{bmatrix}
     k(\mathbf{x}, \mathbf{x}_1)  & \cdots &     
     k(\mathbf{x}, \mathbf{x}_N)
\end{bmatrix}
$
is an $N$-dimensional similarity vector representing pairwise proximities between any instance $\mathbf{x}$ and all the points in the training set $X$. Under the linear assumption \(f(\mathbf{k}) = \mathbf{k}\mathbf{W}\), where \(\mathbf{W} \in \mathbb{R}^{N \times d}\) is a projection matrix to be determined, we directly define \(\mathbf{W}\) in the context of a regression task~\cite{Gisbrecht2012oos, Gisbrecht2015kernel-tsne} by minimizing the least-squares error
\begin{equation}\label{eq:least-squares}
    \sum_{i=1}^N \|\mathbf{z}_i^G - f(\mathbf{k}_i)\|_2^2,
\end{equation}
which yields the explicit solution
$\mathbf{W} = \mathbf{K}^{-1}\mathbf{Y}$,
where \(\mathbf{K}^{-1}\) refers to the pseudo-inverse of the training Gram matrix $\mathbf{K}=[k(\mathbf{x}_i, \mathbf{x}_j)]_{1\leq i,j \leq N}$, and \(\mathbf{Y} = \begin{bmatrix}
    \mathbf{z}_1^G & \cdots & \mathbf{z}_N^G
\end{bmatrix}^T\) contains the precomputed training manifold embeddings. In particular, for manifold learning algorithms that directly assign the low-dimensional coordinates from the eigenvectors of $\mathbf{K}$, e.g., Laplacian Eigenmaps~\cite{belkin2001laplacian}, we have the well-known Nyström formula \(\mathbf{W} = \mathbf{U}\mathbf{\Lambda}^{-1}\)~\cite{bengio2003out, Arias2007connecting, Chen2013sparse}, where \(\mathbf{\Lambda}^{-1} = \text{diag}\left(  \lambda_1^{-1}, \hdots, \lambda_d^{-1} \right)\). Here, \(\lambda_i\) are the $d$ largest (or smallest, depending on the method) eigenvalues of $\mathbf{K}$, and \(\mathbf{U}\) is the matrix whose columns are the corresponding eigenvectors. In Locally Linear Embedding~\cite{saul2003think} and PHATE~\cite{Moon2019phate}, the authors suggested a default OOS extension through linear reconstruction $\mathbf{W}=\begin{bmatrix}
    \mathbf{z}_1^G & \cdots & \mathbf{z}_N^G
\end{bmatrix}^T$. In diffusion-based methods, this provides an alternative means to compress the diffusion process through the training landmarks and has been shown to outperform a direct application of the Nyström extension to diffusion maps~\cite{gigante2019compressed}.

Unlike parametric extensions discussed in Section~\ref{subsec:multitask_autoencoders}, kernel extensions learn an explicit embedding function using kernel representations rather than the original representations in the feature space. While kernel methods are powerful for handling high-dimensional datasets, they require computing pairwise kernels for all points in the training set, which can become computationally expensive for large datasets. In such cases, feature mappings offer greater scalability. However, kernel extensions have been shown to outperform direct parametric methods in unsupervised OOS visualization and classification tasks~\cite{Gisbrecht2015kernel-tsne, Ran2024kumap}. Additionally, in supervised visualization, a carefully chosen kernel mapping can effectively filter out irrelevant features, whereas feature mappings treat all features equally, potentially increasing sensitivity to noisy datasets. Therefore, we align our OOS extension framework with kernel extensions rather than direct parametric methods to stabilize supervised manifold learning, as empirically demonstrated in Appendix~\ref{sec:kernel_vs_feature} on an artificial tree (Appendix~\ref{sec:artificial_tree}).

\section{Methods}
\label{sec:methods}

While traditional kernel extensions have demonstrated computational benefits, they are limited to linear kernel mappings and are primarily designed for unsupervised data visualization or classification using either unsupervised or class-conditional kernel functions~\cite{Gisbrecht2015kernel-tsne, Ran2024kumap}. In this work, we expand the search space of the standard least-squares minimization problem in (\ref{eq:least-squares}) to include general, potentially nonlinear kernel mapping functions \(f\). We also propose a supervised kernel mapping based on Random Forests, specifically tailored for supervised data visualization. Building on the previous Section~\ref{subsec:multitask_autoencoders}, we add a geometrically motivated regularizer to this regression task within a multi-task autoencoder framework.

In this section, we elaborate on the methodology related to our RF-AE framework to extend any kernel-based dimensionality reduction method with Random Forests and autoencoders for supervised data visualization. Specifically, we explain how we combined RF-GAP proximities and the visualization strengths of RF-PHATE with the flexibility of autoencoders to develop a new parametric supervised embedding method, Random Forest Autoencoders (RF-AE). Additionally, we define our evaluation metrics and provide a detailed description of our experimental setup, including the models and datasets used in our experiments and our hyperparameter selection.

\subsection{Extended RF-GAP kernel function}\label{subsec:oosRFGAP}

The RF-GAP proximity~\cite{rhodes2023geometry} between (possibly unseen) instance $\mathbf{x}_i$ and training instance $\mathbf{x}_j$ is 
\begin{equation*}
    p(\mathbf{x}_i, \mathbf{x}_j)=\dfrac{1}{\left|S_{i}\right|} \sum_{t \in S_{i}}  \frac{c_j(t) \cdot I\left(j \in J_{i}(t)\right)}{\left|M_{i}(t)\right|},
\end{equation*}
where $S_i$ denotes the set of out-of-bag trees for observation $\mathbf{x}_i$, $c_j(t)$ is the number of in-bag repetitions for observation $\mathbf{x}_j$ in tree $t$, $I(\cdot)$ is the indicator function, $J_i(t)$ is the set of in-bag points residing in the terminal node of observation $\mathbf{x}_i$ in tree $t$, and $M_i(t)$ is the multiset of in-bag observation indices, including repetitions, co-occurring in a terminal node with $\mathbf{x}_i$ in tree $t$. Note that this definition naturally extends to OOS observations $\mathbf{x}_o \notin X$, which can be treated as out-of-bag for all trees. However, this definition requires that self-similarity be zero, that is, $p(\mathbf{x}_i, \mathbf{x}_i) = 0$. This is not suitable as a similarity measure in some applications. Due to the scale of the proximities---the rows sum to one~\cite{rhodes2023geometry}, so the proximity values are all near zero for larger datasets---, it is not practical to re-assign self-similarities to one. Otherwise, self-similarity would carry equal weight to the combined significance of all other similarities. Instead, we assign values by, in essence, passing down an identical OOB point to all trees where the given observation is in-bag. That is, we define self-similarity as
\begin{equation*}
    p(\mathbf{x}_i, \mathbf{x}_i)=\dfrac{1}{\left|\bar{S}_{i}\right|} \sum_{t \in \bar{S}_{i}}  \frac{c_i(t)}{\left|M_{i}(t)\right|},
\end{equation*}
where $\left|\bar{S}_{i}\right|$ is the set of trees for which $\mathbf{x}_i$ is in-bag. Under this formulation, $p(\mathbf{x}_i, \mathbf{x}_i)$ is on a scale more similar to other proximity values, and Proposition ~\ref{thm:maximality} (Appendix~\ref{sec:proof_maximality}) guarantees that, on average, $p(\mathbf{x}_i, \mathbf{x}_i) > p(\mathbf{x}_i, \mathbf{x}_j)$. Now, we define the row-normalized RF-GAP similarity between a pair of training instances $\mathbf{x}_i$ and $\mathbf{x}_j$ as
\begin{equation}
    \tilde{p}(\mathbf{x}_i, \mathbf{x}_j) = \frac{p'(\mathbf{x}_i, \mathbf{x}_j)}{\sum_{j=1}^{N} p'(\mathbf{x}_i, \mathbf{x}_j)}
    \label{eq:proximity_final}
\end{equation}
where $p'(\mathbf{x}_i, \mathbf{x}_j)=(p(\mathbf{x}_i, \mathbf{x}_j) + p(\mathbf{x}_j, \mathbf{x}_i))/2$. In this definition, we intentionally symmetrized the original RF-GAP similarities through $p'(\cdot, \cdot)$ to use them as input for our prototype selection approach in Section~\ref{subsec:prototype_selection} and applied row-normalization to restore the sum to one property and refocus on the underlying geometry rather than sample distribution.

\subsection{RF-AE architecture}
\label{subsec:rfae_arch}
\begin{figure}[ht]
    \centering
    \includegraphics[width = 0.8\textwidth]{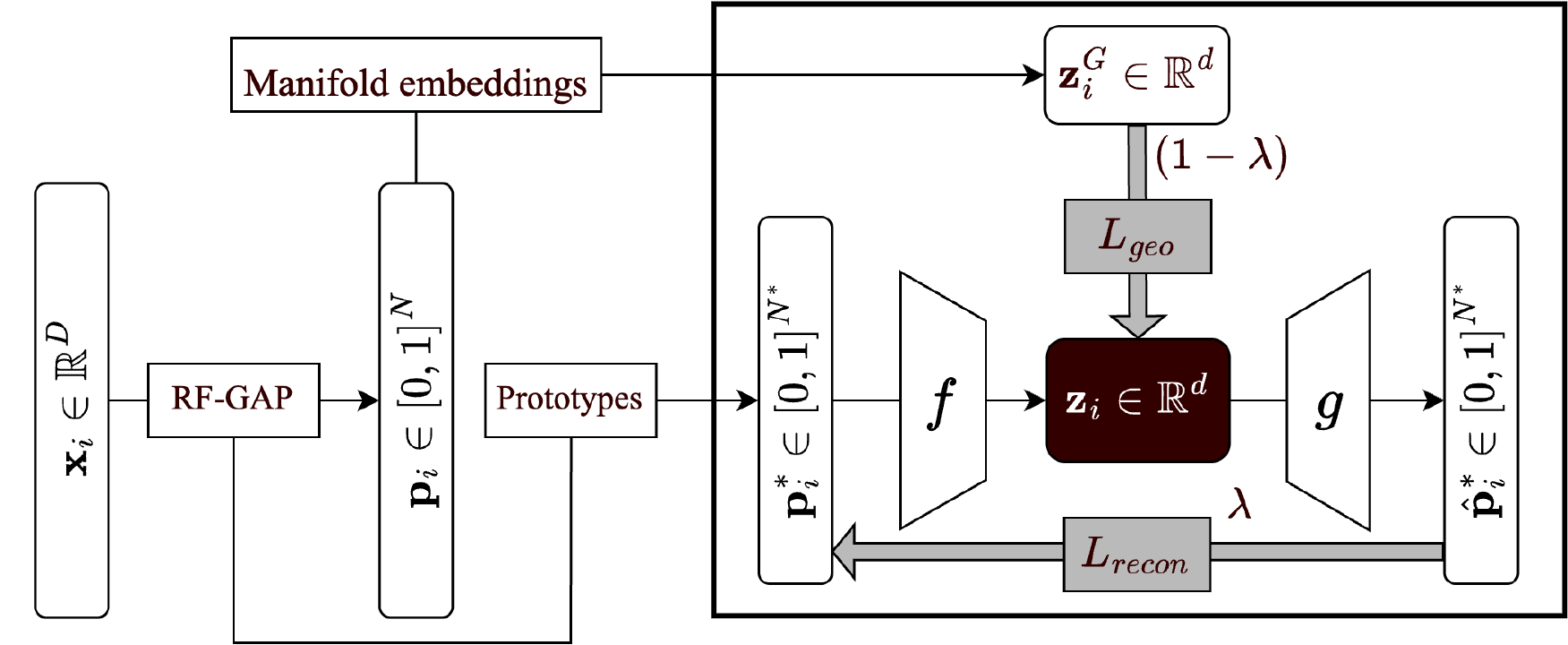}
    \caption{RF-AE architecture with prototype selection and geometric regularization. First, the original feature vectors $\mathbf{x}_i$ are transformed into one-step transition probability vectors $\mathbf{p}_i$ derived from RF-GAP proximities (Section~\ref{subsec:oosRFGAP}). They are further reduced into lower-dimensional vectors $\mathbf{p}^*_i$ that represent transition probabilities to $N^*\ll N$ selected prototypes (Section~\ref{subsec:prototype_selection}). Meanwhile, manifold embeddings $\mathbf{z}_i^G$ are generated using RF-PHATE from the $\mathbf{p}_i$. Finally, $\mathbf{p}^*_i$ and $\mathbf{z}_i^G$ serve as input to the network within the enclosing box, training an encoder $f$ and a decoder $g$ by simultaneously minimizing the reconstruction loss $L_{recon}$ and the geometric loss $L_{geo}$ defined in Section~\ref{subsec:rfae_arch}.}
    \label{fig:rfae_arch}
\end{figure}

To leverage the knowledge gained from an RF model, we modify the traditional AE architecture to incorporate the RF's learning. The forest-generated proximity measures~\cite{rhodes2023geometry}, which indicate similarities between data points relative to the supervised task, serve as a foundation for extending the embedding while integrating the insights acquired through the RF's learning process. In  RF-AE, the original input vectors $\mathbf{x}_i\in \mathbb{R}^D$ used in the vanilla AE are now replaced with the symmetric, row-normalized RF-GAP proximity vector between training instance $\mathbf{x}_i$ and all the other training instances, including itself. That is, each input $\mathbf{x}_i$ used for training is now represented as an $N$-dimensional vector $\mathbf{p}_i$ encoding local-to-global supervised neighbourhood information around $\mathbf{x}_i$, defined using Eq.~\ref{eq:proximity_final}:
\begin{equation*}
  \mathbf{p}_i = \begin{bmatrix}
    \tilde{p}(\mathbf{x}_i, \mathbf{x}_1) & \cdots & \tilde{p}(\mathbf{x}_i, \mathbf{x}_N)\end{bmatrix}\in [0,1]^N.  
\end{equation*}
    
Since its elements sum to one, $\mathbf{p}_i$ contains one-step transition probabilities from training observation with index $i$ to its supervised neighbors indexed $j=1,\hdots,N$ derived from the RF-GAP proximities. Thus, the encoder $f(\mathbf{p}_i)=\mathbf{z}_i\in \mathbb{R}^d$ and decoder $g(\mathbf{z}_i)=\hat{\mathbf{p}}_i$ of the unconstrained RF-AE network are trained through stochastic gradient descent by minimizing the reconstruction loss $L(f, g) = \frac{1}{N}\sum_{i=1}^N L_{recon}(\mathbf{p}_i, \hat{\mathbf{p}}_i)$. Given a learned set of low-dimensional manifold embeddings $G=\{\mathbf{z}^G_i\in \mathbb{R}^d\mid i=1,\hdots, N\}$ (e.g. obtained from RF-PHATE), we additionally force the RF-AE to learn a latent representation $\mathbf{z}_i$ similar to its precomputed counterpart $\mathbf{z}_i^G$ via an explicit geometric constraint to the bottleneck layer, similar to GRAE~\cite{duque2022geometry}. This translates into an added term in the loss formulation, which now takes the form:
\begin{equation*}
    L(f, g) = \frac{1}{N} \sum_{i=1}^N \Big[ 
    \lambda L_{recon}(\mathbf{p}_i, \hat{\mathbf{p}}_i) 
    + (1 - \lambda) L_{geo}(\mathbf{z}_i, \mathbf{z}_i^G) 
\Big].
\end{equation*}
The parameter $\lambda \in [0,1]$ controls the degree to which the precomputed embedding is used in encoding $\mathbf{x}_i$: $\lambda = 1$ is our vanilla RF-AE model without geometric regularization, while $\lambda=0$ reproduces $\mathbf{z}_i^G$ as in the standard kernel mapping formulation. 
We use standard Euclidean distance for the geometric loss to align with the least-squares formulation in Eq. (\ref{eq:least-squares}). While one could define the reconstruction loss as the squared Euclidean distance between input vectors $\mathbf{p}_i$ and their reconstructions, this biases learning toward zero-valued entries, which dominate in large datasets but carry little structural meaning. In contrast, nonzero entries reflect meaningful links in the RF-GAP graph. Although re-weighting the loss to emphasize nonzeros is possible~\cite{Wang2016sdne}, it introduces extra hyperparameters. Instead, we treat $\mathbf{p}_i$ and its reconstruction $\hat{\mathbf{p}}_i = (g\circ f)(\mathbf{p}_i)$ as probability distributions and use the Jensen-Shannon Divergence (JSD)~\cite{lin1991divergence} as the reconstruction loss:
\begin{equation*}
L_{recon}(\mathbf{p}_i, \hat{\mathbf{p}}_i) = \mathrm{JSD}(\mathbf{p}_i \,\|\,\hat{\mathbf{p}}_i), \qquad
L_{geo}(\mathbf{z}_i,\mathbf{z}_i^G) = \| \mathbf{z}_i - \mathbf{z}_i^G \|_2^2.
\end{equation*}
The JSD promotes latent representations that reconstruct both local and global RF-GAP neighborhoods~\cite{Im_Verma_Branson_2018}. In this work, we set the latent dimension $d=2$ to emphasize on visual interpretability. We use RF-PHATE as the geometric constraint due to its effectiveness in supervised data visualization~\cite{rhodes2021rfphate, rhodes2024gaining}, although any dimensionality reduction method can be extended this way. Moreover, as RF-PHATE already encodes multiscale information, combining it with JSD reconstruction further guides learning toward geometrically meaningful representations while supporting global consistency. Refer to Fig.~\ref{fig:rfae_arch} for a comprehensive illustration of our RF-AE architecture.

\subsection{Input dimensionality reduction with class-wise prototype selection}\label{subsec:prototype_selection}
The input dimensionality of our RF-AE architecture scales with the training size $N$, which may cause memory issues during GPU-optimized training when dealing with large training sets. Thus we further reduce the input dimensionality of $\mathbf{p}_i$ from $N$ to $N^*\ll N$ by selecting $N^*$ prototypes. The prototypes are selected using uniform
class-wise $k$-medoids~\cite{gomes2010budgeted, Tan_Soloviev_Hooker_Wells_2020} on the induced RF-GAP training dissimilarities. First, we max-normalize the symmetrized RF-GAP proximities to form the symmetric dissimilarity matrix $\begin{bmatrix}
    \max_{u,v}\left\{p'(\mathbf{x}_u, \mathbf{x}_v)\right\} - p'(\mathbf{x}_i, \mathbf{x}_j)
\end{bmatrix}\in [0,1]^{N\times N}$. Then, for a dataset with $q$ classes, we find $k= N^*/q$-medoids for each class using their corresponding RF-GAP dissimilarities as input to FasterPAM~\cite{schubert2019faster, schubert2021fast}.
Let $\mathfrak{M}=\{\mathfrak{m}_1, \ldots, \mathfrak{m}_{N^*}\}$ denote the resulting set of medoid indices. Then instead of using RF-GAP transition probabilities from any point $i$ to every training point $j$ as before, we form RF-GAP transition probabilities from any point $i$ to each prototype $j\in\mathfrak{M}$ as 
\begin{equation*}
    \mathbf{p}^*_i = \begin{bmatrix}
    \tilde{p}^*(\mathbf{x}_i, \mathbf{x}_{\mathfrak{m}_1}) & \cdots & \tilde{p}^*(\mathbf{x}_i, \mathbf{x}_{\mathfrak{m}_{N^*}}) 
    \end{bmatrix}, \quad
    \tilde{p}^*(\mathbf{x}_i, \mathbf{x}_j) = \frac{p'(\mathbf{x}_i, \mathbf{x}_j)}{\sum_{j \in \mathfrak{M}} p'(\mathbf{x}_i, \mathbf{x}_j)}.
\end{equation*}
Fig.~\ref{fig:rfae_arch} contextualizes this prototype selection mechanism within our RF-AE architecture. We also note that using prototypes allows for faster OOS projections since we no longer need to compute RF-GAP proximities to all training points.



\subsection{Quantifying supervised OOS embedding fit}\label{subsec:quantify_oos_embedding}

Beyond standard $k$-NN accuracy~\cite{maaten2009parametric_tsne, sainburg2021parametric_umap, Ghosh2022centroid, goldberber2004nca, wang2021understanding}, which evaluates class separability in the embedding space, it is equally important to assess how well the embedding preserves the structure of informative features. Without this, class-conditional methods that artificially inflate separation may be favored, even if they distort meaningful feature–label relationships. Conversely, purely unsupervised criteria—such as neighbor preservation~\cite{sainburg2021parametric_umap} or global distance correlation~\cite{Kobak2019tsne}—can undervalue supervised models that discard irrelevant features aligned with the classification task.

Inspired by Rhodes et al.~\cite{rhodes2024gaining}, we formalize \textit{structural importance alignment}, which quantifies the correlation between feature importances for classification and for structure preservation. Given a training/test split \(X = X_{\mathrm{train}} \cup X_{\mathrm{test}}\) with labels \(Y = Y_{\mathrm{train}} \cup Y_{\mathrm{test}}\), and embeddings
$
f_{\mathrm{emb}}(X) = f_{\mathrm{emb}}(X_{\mathrm{train}}) \cup f_{\mathrm{emb}}(X_{\mathrm{test}}) =Z_{\mathrm{train}} \cup Z_{\mathrm{test}}
$ from a trained encoder $f_{\mathrm{emb}}$, we define test–train distance matrices in the original and embedded spaces as:
\[
\mathbf{D}_{\mathrm{test}}[i, j] = \| \mathbf{x}_i^{\mathrm{test}} - \mathbf{x}_j^{\mathrm{train}} \|_2, \quad
\mathbf{D}_{\mathrm{test}}^{\mathrm{emb}}[i, j] = \| \mathbf{z}_i^{\mathrm{test}} - \mathbf{z}_j^{\mathrm{train}} \|_2, \quad
\mathbf{D}_{\mathrm{test}}, \mathbf{D}_{\mathrm{test}}^{\mathrm{emb}} \in \mathbb{R}_{+}^{N_{\mathrm{test}} \times N_{\mathrm{train}}}.
\]

\textbf{Classification importances} are computed using a user-defined classifier \(f_{\mathrm{cls}} : \mathbb{R}^D \to \mathcal{Y}\) trained on \(X_{\mathrm{train}}\). Let \(\mathrm{acc}_{f_{\mathrm{cls}}}(X_{\mathrm{test}}, Y_{\mathrm{test}})\) denote its accuracy on the test set. Then, the importance of feature \(i\) is:
\[
\mathcal{C}_i = \mathrm{acc}_{f_{\mathrm{cls}}}(X_{\mathrm{test}}, Y_{\mathrm{test}}) - \mathrm{acc}_{f_{\mathrm{cls}}}(\widetilde{X}_{\mathrm{test}}^{(i)}, Y_{\mathrm{test}}),
\]
where \(\widetilde{X}_{\mathrm{test}}^{(i)}\) is the perturbed test set in which feature \(i\) and its correlated features are permuted across samples (see Algorithm~\ref{alg:sampling} in Appendix~\ref{sec:sampling}).

\textbf{Structural importances} are computed using an unsupervised \emph{structure preservation score} \( s(\cdot, \cdot) \) that quantifies how well an embedding preserves pairwise relationships from the original space. Higher scores indicate better preservation of structure. We consider several commonly used definitions of \(s\), including local scores \(s \in \{\textit{QNX}, \textit{Trust}\}\) and global scores \(s \in \{\textit{Spear}, \textit{Pearson}\}\)~\cite{sainburg2021parametric_umap, Kobak2019tsne, kobak2019heavy, venna2006local, Gildenblat_Pahnke_2025}. Full definitions are provided in Appendix~\ref{subsec:def_structure_preservation}.

Given a test set \( \mathbf{D}_{\mathrm{test}} \) and its embedding \( \mathbf{D}_{\mathrm{test}}^{\mathrm{emb}} \), the importance of feature \(i\) is then defined as:
\[
\mathcal{S}_i = s(\mathbf{D}_{\mathrm{test}}, \mathbf{D}_{\mathrm{test}}^{\mathrm{emb}}) - s(\widetilde{\mathbf{D}}_{\mathrm{test}}^{(i)}, \mathbf{D}_{\mathrm{test}}^{\mathrm{emb}}),
\]
where \(\widetilde{\mathbf{D}}_{\mathrm{test}}^{(i)}\) is the perturbed distance matrix obtained by replacing feature \(i\) in \(X_{\mathrm{test}}\) with noise (Algorithm ~\ref{alg:sampling}), while holding \(X_{\mathrm{train}}\) fixed. A larger drop in \(s\) indicates that the OOS embedding relies more heavily on the structure induced by feature \(i\).

To assess whether the embedding structure supports classification-relevant features, we compute the alignment between structural importances \(\mathcal{S} = \{ \mathcal{S}_1, \ldots, \mathcal{S}_D \}\) and classification importances \(\mathcal{C} = \{ \mathcal{C}_1, \ldots, \mathcal{C}_D \}\) using the Kendall rank correlation coefficient \(\tau(\mathcal{C}, \mathcal{S}) \in [-1, 1]\)~\cite{Kendall1938}. Higher values indicate that the embedding prioritizes features most relevant to the classification task. Fig.~\ref{fig:sia_illustration} (Appendix~\ref{sec:sia_illustration}) illustrates this Structural Importance Alignment (SIA) framework.

Note that SIA depends on both the choice of classifier \( f_{\mathrm{cls}} \) and structure score \( s \). For \( f_{\mathrm{cls}} \), we use an ensemble with equal-weight majority voting across \(k\)-NN, SVM, and MLP classifiers to reduce model-specific bias (see Appendix~\ref{sec:rfae_exp_setting} for hyperparameters). Each dataset achieves at least 60\% accuracy (Appendix~\ref{sec:baseline_cls_perf}). For \(s\), we report four variants of SIA based on the chosen structure score, capturing both local and global structure preservation.

\section{Results}

\subsection{RF-AE balances structural importance alignment and class separability}\label{subsec:quant_comp}

We assessed the trade-off between SIA and $k$-NN classification accuracy achieved by RF-AE against several baseline methods across 20 datasets spanning diverse domains. Each dataset contained a minimum of 1,000 samples and at least 10 features. Training and OOS embeddings were generated using an 80/20 stratified train/test split, except for Isolet, Landsat Satellite, Optical Digits, USPS, HAR, OrganC MNIST and Blood MNIST, where predefined splits were used. We applied min-max normalization to the input features prior to training and inference. 
Detailed descriptions of the datasets are provided in Appendix~\ref{sec:data}.

Table~\ref{tab:quantitative_comp} shows average local SIA and $k$-NN accuracy scores across 20 datasets and 10 repetitions. We report separate local ($s=\textit{QNX},\textit{Trust}$) and global ($s=\textit{Spear},\textit{Pearson}$) SIA scores. Accuracy is averaged over $k = 5$ to $\sqrt{N_{\mathrm{train}}}$ (in steps of 10) to better reflect global class separability and penalize class fragmentation. We compared RF-AE with \( \lambda = 0.01 \) and $N^*=0.1 N_{\mathrm{train}}$ to 13 baselines, including the default RF-PHATE linear kernel extension~\cite{Moon2019phate} (Section~\ref{subsec:kernel_methods}), vanilla AE, principal component analysis (PCA), supervised PCA, parametric \( t \)-SNE (P-TSNE ~\cite{maaten2009parametric_tsne, damrich2023from}), parametric UMAP (P-UMAP~\cite{sainburg2021parametric_umap, damrich2023from}), parametric supervised UMAP (P-SUMAP~\cite{sainburg2021parametric_umap}), pairwise controlled manifold approximation projection (PACMAP~\cite{wang2021understanding}), CE~\cite{Ghosh2022centroid}, CEBRA~\cite{Schneider2023cebra}, SSNP~\cite{espadoto2021ssnp} using ground-truth labels, neighborhood component analysis (NCA~\cite{goldberber2004nca}), and partial least squares discriminant analysis (PLS-DA~\cite{gottfries1995diagnosis, barker2003partial}). All externally sourced methods were run using their default hyperparameter settings, as specified in the original implementations. See Appendix~\ref{subsec:exp_setting} for full experimental details. The compute resources required for the experiments include a GPU with at least 40 GB of memory and a CPU with 128 GB of RAM, further details are provided in Appendix~\ref{subsec:exp-compute-resource}.

\begin{table}[ht]
\caption{Local ($s=\textit{QNX},\textit{Trust}$) and global ($s=\textit{Spear},\textit{Pearson}$) SIA scores, along with test $k$-NN accuracies for our RF-AE method and 13 baselines. Scores are shown as mean $\pm$ std across 20 datasets and 10 repetitions. Methods are sorted according to accuracy. Top three scores in each metric are highlighted with underlined bold italics (first), bold italics (second), and italics (third). In the case of ties, methods are further ranked by their standard deviations. Supervised methods are marked by an asterisk.}
\label{tab:quantitative_comp}

\centering
\small
\begin{sc}
\begin{tabular}{lccccc}
\toprule
 & \multicolumn{2}{c}{Local SIA} & \multicolumn{2}{c}{Global SIA} &  \\
 \cmidrule(r){2-5}
 & QNX & Trust & Spear & Pearson & \multicolumn{1}{c}{$k$-NN Acc} \\
\midrule
RF-AE* & \textbf{\textit{0.800 ± 0.025}} & \textbf{\textit{0.818 ± 0.025}} & \underline{\textbf{\textit{0.776 ± 0.051}}} & \underline{\textbf{\textit{0.776 ± 0.050}}} & \underline{\textbf{\textit{0.831 ± 0.011}}} \\
SSNP* & 0.760 ± 0.047 & 0.772 ± 0.045 & 0.685 ± 0.089 & 0.694 ± 0.080 & \textbf{\textit{0.809 ± 0.030}} \\
P-SUMAP* & 0.756 ± 0.028 & 0.768 ± 0.025 & 0.647 ± 0.048 & 0.647 ± 0.048 & \textit{0.797 ± 0.011} \\
CE* & 0.795 ± 0.050 & \textit{0.818 ± 0.051} & \textit{0.765 ± 0.051} & \textbf{\textit{0.763 ± 0.054}} & 0.797 ± 0.043 \\
RF-PHATE* & \textit{0.798 ± 0.027} & \underline{\textbf{\textit{0.823 ± 0.025}}} & 0.749 ± 0.040 & 0.750 ± 0.043 & 0.794 ± 0.009 \\
NCA* & \underline{\textbf{\textit{0.808 ± 0.027}}} & 0.805 ± 0.025 & \textbf{\textit{0.771 ± 0.032}} & \textit{0.759 ± 0.033} & 0.760 ± 0.007 \\
PACMAP & 0.749 ± 0.026 & 0.758 ± 0.025 & 0.688 ± 0.029 & 0.688 ± 0.029 & 0.743 ± 0.011 \\
P-TSNE & 0.743 ± 0.028 & 0.747 ± 0.028 & 0.684 ± 0.036 & 0.666 ± 0.038 & 0.712 ± 0.012 \\
AE & 0.744 ± 0.027 & 0.751 ± 0.029 & 0.695 ± 0.044 & 0.655 ± 0.053 & 0.700 ± 0.018 \\
P-UMAP & 0.757 ± 0.027 & 0.744 ± 0.028 & 0.674 ± 0.035 & 0.657 ± 0.038 & 0.655 ± 0.022 \\
SPCA* & 0.767 ± 0.026 & 0.759 ± 0.030 & 0.741 ± 0.031 & 0.738 ± 0.032 & 0.624 ± 0.009 \\
PLS-DA* & 0.715 ± 0.026 & 0.708 ± 0.028 & 0.659 ± 0.027 & 0.639 ± 0.028 & 0.592 ± 0.009 \\
CEBRA* & 0.780 ± 0.045 & 0.775 ± 0.050 & 0.735 ± 0.062 & 0.728 ± 0.068 & 0.582 ± 0.040 \\
PCA & 0.745 ± 0.027 & 0.742 ± 0.026 & 0.733 ± 0.027 & 0.727 ± 0.028 & 0.563 ± 0.009 \\
\bottomrule
\end{tabular}
\end{sc}
\end{table}

Unsurprisingly, unsupervised methods generally rank lower than supervised approaches in terms of $k$-NN classification accuracy. However, even among high-accuracy models such as SSNP and P-SUMAP, we observe a notable drop in local and global SIA scores. This suggests an overemphasis on class separability at the expense of preserving the underlying supervised structure—an indication of structural distortion. Unsupervised methods, on the other hand, also struggle with SIA metrics, which is expected given their objective to preserve unsupervised pairwise similarities that may be influenced by irrelevant or noisy features. In contrast, RF-AE achieves the highest $k$-NN accuracy by a substantial margin, while consistently ranking in the top two across both local and global SIA scores. This demonstrates RF-AE’s ability to not only ensure class separation but also preserve meaningful supervised relationships, effectively emphasizing the most relevant features for the classification task.

RF-AE maintains strong performance across all evaluation metrics when varying $\lambda$, as shown in Table~\ref{tab:ablation_lam}. 
Additional ablation studies (Appendix~\ref{sec:ablation_supplemental}) confirm that RF-AE’s performance is robust to both $\lambda$ and prototype count $N^*$. Its superiority in SIA also persists under varying classification importance strategies (Appendix~\ref{sec:sia_varying_imp}), reflecting better alignment with the ground-truth feature importance hierarchy.

\begin{table}[ht]
\caption{Comparison of SIA scores and $k$-NN accuracy across different $\lambda$ values for RF-AE. Each score is compared with baseline models in Table~\ref{tab:quantitative_comp}, and highlighted only if it ranks among the top three overall. Top three values per metric are highlighted using underlined bold italics (first), bold italics (second), and italics (third).}
\label{tab:ablation_lam}

\centering
\small
\begin{sc}
\begin{tabular}{lccccc}
\toprule
 & \multicolumn{2}{c}{Local SIA} & \multicolumn{2}{c}{Global SIA} &  \\
 \cmidrule(r){2-5}
 & QNX & Trust & Spear & Pearson & \multicolumn{1}{c}{$k$-NN Acc} \\
\midrule
$\lambda = 0$ & \textbf{\textit{0.803 ± 0.027}} & \textit{0.817 ± 0.025} & 0.747 ± 0.038 & 0.749 ± 0.042 & \textbf{\textit{0.806 ± 0.010}} \\
$\lambda = 0.001$ & \textbf{\textit{0.801 ± 0.025}} & \textbf{\textit{0.819 ± 0.026}} & \textit{0.763 ± 0.040} & \textbf{\textit{0.760 ± 0.044}} & \underline{\textbf{\textit{0.827 ± 0.010}}} \\
$\lambda = 0.01$ & \textbf{\textit{0.800 ± 0.025}} & \textbf{\textit{0.818 ± 0.025}} & \underline{\textbf{\textit{0.776 ± 0.051}}} & \underline{\textbf{\textit{0.776 ± 0.050}}} & \underline{\textbf{\textit{0.831 ± 0.011}}} \\
$\lambda = 0.1$ & \textbf{\textit{0.798 ± 0.026}} & \textit{0.817 ± 0.025} & \underline{\textbf{\textit{0.785 ± 0.045}}} & \underline{\textbf{\textit{0.788 ± 0.045}}} & \underline{\textbf{\textit{0.831 ± 0.011}}} \\
$\lambda = 1$ & \textbf{\textit{0.799 ± 0.026}} & \textit{0.816 ± 0.026} & 0.701 ± 0.105 & 0.700 ± 0.108 & \underline{\textbf{\textit{0.837 ± 0.011}}} \\
\bottomrule
\end{tabular}
\end{sc}
\end{table}

\subsection{Qualitative comparison through OOS visualizations}\label{subsec:oos_viz}
We qualitatively assessed the capability of four methods to embed OOS instances on Sign MNIST (A--K) and OrganC MNIST dataset. Each model was trained on the training subset, and we subsequently mapped the test set with the learned encoder. Fig.~\ref{fig:viz_main} depicts the resulting visualizations.

\begin{figure*}[ht]
    \centering
    \includegraphics[width = 1\textwidth]{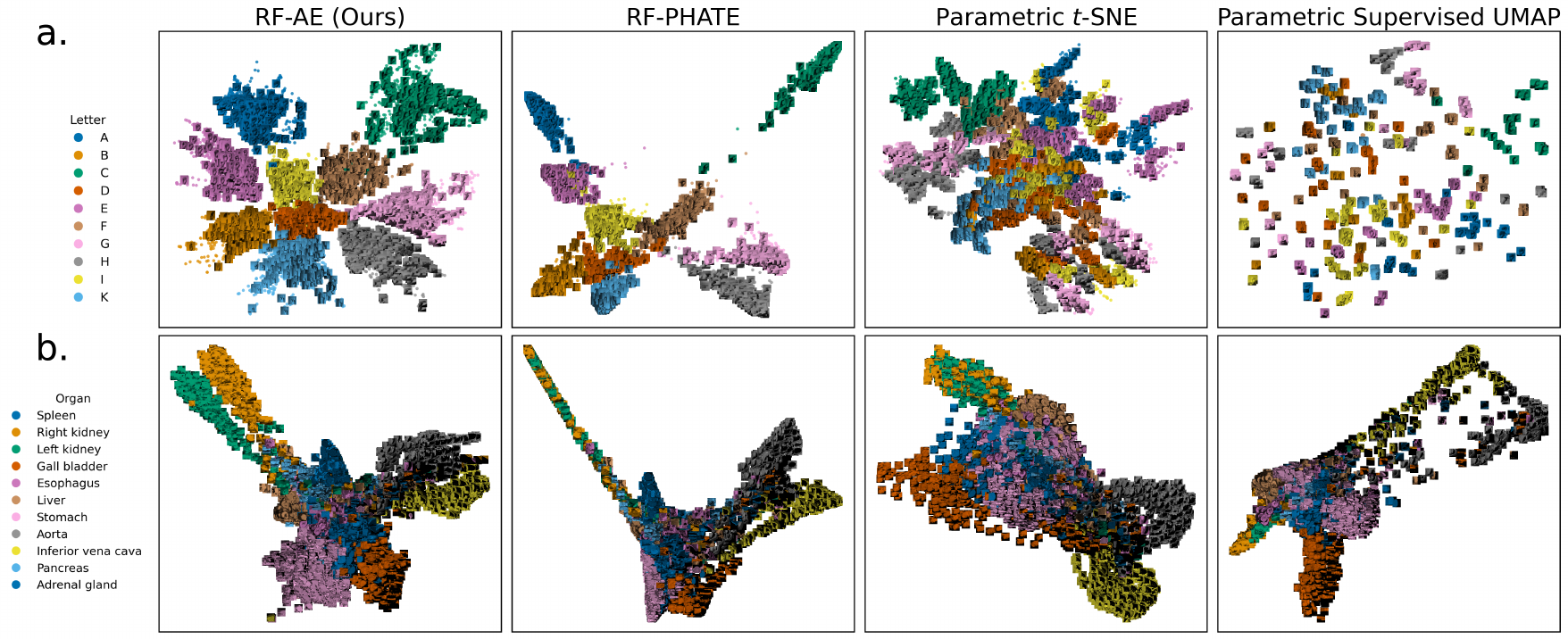}
    \caption{OOS visualization using four different dimensionality reduction methods. Training points are colored by labels, and test points are depicted with original images. Training points are omitted in OrganC MNIST for clarity.
    \textbf{a.} Sign MNIST (A--K) dataset (Table~\ref{tab:rfae_data}): RF-AE captures supervised relationships by preserving class-specific variations, such as shadowing and hand orientation, while also highlighting inter-class similarities and maintaining clear class separability. The default RF-PHATE's kernel extension compresses clusters excessively. Parametric $t$-SNE and parametric supervised UMAP demonstrate sensitivity to irrelevant features. \textbf{b.} OrganC MNIST dataset: RF-AE clearly separates similar organ types while preserving their anatomical proximity, capturing both class identity and biological relevance. Other methods tend to merge these classes or distort their relationships, failing to reflect fine-grained anatomical distinctions. 
    }
    \label{fig:viz_main}
\vskip -0.1in
\end{figure*}

From the Sign MNIST (A--K) plot in Fig.~\ref{fig:viz_main}a, RF-AE retains the overall shape of the RF-PHATE embedding while providing finer within-class detail. In contrast, RF-PHATE's default kernel extension compresses class clusters, making local relationships harder to discern. RF-AE expands these clusters, revealing within-class patterns, such as the logical transition between variations in shadowing and hand orientation to represent the letter ``C'' (top-right cluster). This detail is lost in RF-PHATE's default kernel extension, which over-relies on diffusion and smooths out local differences. P-TSNE captures local structure but frequently fragments same-class samples based on irrelevant background variations—for example, grouping shadowed “G” and “H” instances together (bottom cluster) while separating them from unshadowed counterparts (far-left cluster). RF-AE avoids this issue, distinguishing within-class variations while preserving the class-specific clusters. P-SUMAP is also sensitive to irrelevant features and tends to artificially cluster neighboring points of the same class, leading to a sparse and fragmented embedding.


From the OrganC MNIST plot in Fig.~\ref{fig:viz_main}b, RF-AE yields the most structured and interpretable embedding, forming well-separated clusters while preserving anatomical proximity—e.g., between the left and right kidneys, or among the stomach, liver, and pancreas. RF-PHATE captures smooth global transitions but merges nearby classes like the kidneys and inferior vena cava, reducing local separability. P-TSNE maintains compact local clusters but distorts global structure, leading to fragmented class placement. P-SUMAP, despite higher $k$-NN accuracy than P-TSNE, has the lowest spatial interpretability, with elongated projections and entangled anatomical relationships that obscure class transitions.

This qualitative analysis of the Sign MNIST and OrganC MNIST dataset underscores the importance of regularization and methodological choices in creating meaningful embeddings for supervised tasks. Our RF-AE architecture, enhanced with RF-PHATE regularization, effectively preserves both local and global structures, outperforming existing methods.
The visualizations and analyses of the other models, along with their quantitative comparisons, are provided in Appendix~\ref{sec:oos_viz_supplemental}. 

\section{Discussion}
\label{sec:discussion}
The significance of supervised dimensionality reduction lies in its ability to reveal meaningful relationships between features and labels. As shown  by Rhodes et al.~\cite{rhodes2024gaining}, RF-PHATE stands out as a strong solution for supervised data visualization. However, it lacks an embedding function for OOS extension. To address this limitation, we designed Random Forest Autoencoders (RF-AE), an autoencoder-based architecture that reconstructs RF-GAP neighborhoods while preserving the supervised geometry captured by precomputed RF-PHATE embeddings. Our experiments confirmed the utility of this extension, demonstrating its ability to embed new data points while retaining the intrinsic supervised manifold structure. We quantitatively showed that RF-AE with RF-PHATE regularization outperforms baseline kernel extensions and other parametric embedding models in generating OOS embeddings that preserve feature importances relevant to the classification task while maintaining class separability. We further showed in Appendix~\ref{sec:ablation_supplemental} that RF-AE’s performance is robust to the geometric constraint $\lambda$ and the number of selected prototypes $N^*$. Visually, we observed that RF-AE regularized by RF-PHATE inherits the denoised local-to-global supervised structure of RF-PHATE while increasing resolution for improved within-class visualization. This results in a more effective tradeoff between preserving informative structure and achieving class separability than baseline RF-PHATE kernel extensions. Other methods either over-emphasize class separability or fail to incorporate sufficient supervision, leading to noisier visualizations. Finally, RF-AE’s ability to project unseen data without requiring labels makes it well suited for semi-supervised tasks on large-scale datasets.

Despite its effectiveness, RF-AE inherits scalability limitations due to the computational cost of computing RF-GAP proximities (Appendix~\ref{sec:ablation_supplemental}). Future work will explore adaptive prototype selection~\cite{Tan_Soloviev_Hooker_Wells_2020} and additional efficiency improvements to reduce the cost of RF-GAP calculations on large-scale datasets. Moreover, since the RF-GAP kernel is compatible with any kernel-based dimensionality reduction method, we plan to extend other methods within the RF-AE framework. From a broader perspective (Appendix~\ref{sec:impacts}), RF-AE will support decision-makers by validating expert or AI-driven predictions through structure- and label-informed 2D visualizations.



{
\small

\bibliographystyle{unsrt}         
\bibliography{references}          

}

\newpage
\appendix

\renewcommand{\thefigure}{S\arabic{figure}}
\renewcommand{\thetable}{S\arabic{table}}
\renewcommand{\thealgocf}{S\arabic{algocf}}
\setcounter{figure}{0}
\setcounter{table}{0}
\setcounter{algocf}{0}

\section{Maximality of the RF-GAP self-similarity}\label{sec:proof_maximality}

Recall the RF-GAP proximity (Section~\ref{subsec:oosRFGAP}) between observations $\mathbf{x}_i$ and $\mathbf{x}_j$
\[
p(\mathbf{x}_i,\mathbf{x}_j)
=
\begin{cases}
\displaystyle
\frac{1}{\lvert \bar S_i\rvert}
\sum_{t\in \bar S_i}
\frac{c_i(t)}{\lvert M_i(t)\rvert},
& j = i,\\[1.5em]
\displaystyle
\frac{1}{\lvert S_i\rvert}
\sum_{t\in S_i}
\frac{c_j(t)\,I\left(j\in J_i(t)\right)}{\lvert M_i(t)\rvert},
& j \neq i,
\end{cases}
\]
where:
\begin{itemize}
  \item $S_i$ is the set of trees in which $\mathbf{x}_i$ is \emph{out‐of‐bag} (OOB).
  \item $\bar S_i$ is the set of trees in which $\mathbf{x}_i$ is \emph{in‐bag}.
  \item $c_i(t)$ is the bootstrap multiplicity of $\mathbf{x}_i$ in tree $t$.
  \item $J_i(t)$ is the set of in‐bag points that share the terminal node with $\mathbf{x}_i$ in tree $t$.
  \item $M_i(t)$ is the multiset of in‐bag sample indices in that terminal node (counting multiplicities).
\end{itemize}

\medskip

For any fixed tree $t$, we define the random quantities
\begin{itemize}
  \item $B_i(t) = I(\mathbf{x}_i\text{ is in‐bag in tree }t),$
  
  \item $D_{ij}(t) = I(B_j(t)=1\text{ and }\mathbf{x}_j\text{ shares }\mathbf{x}_i\text{'s leaf in tree }t),$
  
  \item $c_i(t) \sim \mathrm{Binomial}\bigl(N,\tfrac1N\bigr),\quad i=1,\dots,N,$
  

\end{itemize}

Thus, considering all trees in the forest,

\begin{itemize}
      \item $T_i^{\mathrm{IB}} = \sum_{t=1}^{|T|} B_i(t),$
  
      \item $T_i^{\mathrm{OOB}} = \sum_{t=1}^{|T|} 1-B_i(t),$
\end{itemize}

and per‐tree contributions to self‐ and cross‐similarity are re-written as

\[
\alpha_{ii}(t)
:= B_i(t)\,\frac{c_i(t)}{\lvert M_i(t)\rvert},
\qquad
\alpha_{ij}(t)
:= D_{ij}(t)\,\frac{c_j(t)}{\lvert M_i(t)\rvert}
\quad(i\neq j).
\]

Under the standard Random Forests assumptions, the following holds:

\begin{itemize}
\item \emph{Tree independence.}
          Each tree is grown from an independent bootstrap sample and an
          independent sequence of feature splits, ensuring i.i.d.\ per‑tree
          contributions $\alpha_{ii}(t)$ and $\alpha_{ij}(t)$.
  \item \emph{Bootstrap inclusion probability.}  An observation is in‐bag in tree $t$ with probability
\begin{align*}
    p := \mathbb{P}[B_i(t)=1] = \mathbb{P}[c_i(t)\geq 1]=1 - \mathbb{P}[c_i(t)= 0]
    &=1 - \Bigl(1-\tfrac1N\Bigr)^{N}\\
    &\longrightarrow\;1 - e^{-1}\approx0.632.
\end{align*}

Hence,
\[
\begin{array}{rcl}
B_i(t) &\sim& \mathrm{Bernoulli}(p) \\[1ex]
T_i^{\mathrm{IB}} &\sim& \mathrm{Binomial}(|T|,\,p) \\[1ex]
T_i^{\mathrm{OOB}} &\sim& \mathrm{Binomial}(|T|,\,1-p)
\end{array}
\]


  \item \emph{Co‐occurrence probability.}  Even if $\mathbf{x}_j$ is very similar to $\mathbf{x}_i$, the probability that they end up together in the same leaf \emph{and} $\mathbf{x}_i$ was OOB is strictly less than 1:
  \[
    q_{ij}
    :=\mathbb{P}\left[D_{ij}(t)=1\mid B_i(t)=0\right]
    \;<\;1\qquad (i\neq j).
  \]
\end{itemize}

\begin{proposition}\label{thm:maximality}
For every fixed $i$ and any $j\neq i$, in the limit as the number of trees $|T|\to\infty$,
\[
p(\mathbf{x}_i,\mathbf{x}_i)
\;>\;
p(\mathbf{x}_i,\mathbf{x}_j)
\]
\end{proposition}

\begin{proof}
 
RF-GAP similarities are re-written as random variables:
\[
p(\mathbf{x}_i, \mathbf{x}_j) =
\begin{cases}
  \displaystyle \frac{1}{T_i^{\mathrm{IB}}} \sum_{t=1}^T \alpha_{ii}(t), & \text{if } i = j, \\[10pt]
  \displaystyle \frac{1}{T_i^{\mathrm{OOB}}} \sum_{t=1}^T \alpha_{ij}(t), & \text{otherwise}.
\end{cases}
\]

By the Strong Law of Large Numbers and tree‐independence, as \(|T|\to\infty\) we have almost surely
\[
\frac{1}{|T|}\sum_{t=1}^{|T|} \alpha_{ii}(t)
\;\longrightarrow\;
\mathbb{E}[\alpha_{ii}(t)]
\;=\;
\mathbb{E}\left[B_i(t)\,\tfrac{c_i(t)}{|M_i(t)|}\right]
\;=\;p\;\underbrace{\mathbb{E}\left[\tfrac{c_i(t)}{|M_i(t)|}\mid B_i(t)=1\right]}_{=\mu}
\;=\;p\,\mu,
\]

\begin{align*}
\frac{1}{|T|}\sum_{t=1}^{|T|} \alpha_{ij}(t)
\;\longrightarrow\;
\mathbb{E}[\alpha_{ij}(t)]
\;&=\;
\mathbb{E}\left[D_{ij}(t)\,\tfrac{c_j(t)}{|M_i(t)|}\right]\\
\;&=\;
(1-p)\,q_{ij}\,\underbrace{\mathbb{E}\left[\tfrac{c_j(t)}{|M_i(t)|}\bigm|D_{ij}(t)=1,B_i(t)=0\right]}_{\le\mu}\\
\;&\le\;(1-p)\,q_{ij}\,\mu.
\end{align*}

The inequality $\mathbb{E}\left[\left.\frac{c_j(t)}{|M_i(t)|}\right| D_{ij}(t)=1, B_i(t)=0 \right] \le \mu := \mathbb{E}\left[\left.\frac{c_i(t)}{|M_i(t)|}\right| B_i(t)=1 \right]$ follows because the marginal distributions of $c_j(t)$ and $c_i(t)$ are identical, but the conditional distribution of $|M_i(t)|$ is stochastically larger under $D_{ij}(t)=1, B_i(t)=0$ than under $B_i(t)=1$. Indeed, conditioning on $D_{ij}(t)=1$ requires that $\mathbf{x}_j$ shares a leaf with $\mathbf{x}_i$, even though $\mathbf{x}_i$ is OOB, which biases the leaf size upward — thereby lowering the expected normalized count $c_j(t)/|M_i(t)|$.
Moreover, almost surely,
\[
\frac{T_i^{\mathrm{IB}}}{|T|}\;\longrightarrow\;\mathbb{E}\left[B_i(t)\right]=p,
\qquad
\frac{T_i^{\mathrm{OOB}}}{|T|}\;\longrightarrow\;\mathbb{E}\left[1-B_i(t)\right]=1-p
\]
Thus,
\begin{align*}
p(\mathbf{x}_i,\mathbf{x}_i)
&=
\frac{\frac{1}{|T|}\sum_{t=1}^{|T|}\alpha_{ii}(t)}{\tfrac{T_i^{\mathrm{IB}}}{|T|}}
\;\longrightarrow\;
\frac{p\,\mu}{p}
= \mu, \\[1em]
p(\mathbf{x}_i,\mathbf{x}_j)
&=
\frac{\frac{1}{|T|}\sum_{t=1}^{|T|}\alpha_{ij}(t)}{\tfrac{T_i^{\mathrm{OOB}}}{|T|}}
\;\longrightarrow\;
\frac{\mathbb{E}[\alpha_{ij}(t)]}{1-p}
\;\leq\;
\frac{(1-p)\,q_{ij}\,\mu}{1-p}
= q_{ij}\,\mu.
\end{align*}
Since we assumed \(q_{ij}<1\), it follows that
\[
\mu
\;>\;
q_{ij}\,\mu
\quad\Longrightarrow\quad
\lim_{|T|\to\infty}p(\mathbf{x}_i,\mathbf{x}_i)
\;>\;
\lim_{|T|\to\infty}p(\mathbf{x}_i,\mathbf{x}_j).
\]
\end{proof}

\begin{remark}
    Finite-\(|T|\) concentration bounds (e.g.\ Hoeffding’s inequality~\cite{Hoeffding_1963}) imply the same inequality holds with overwhelming probability.
\end{remark}
\begin{remark}
    The assumption $q_{ij}<1$ is not necessary for non-strict inequality.
\end{remark}

\section{RF-GAP representations stabilize supervised manifold learning}\label{sec:kernel_vs_feature}

We designed our RF-AE framework under the premise that encoders operating on (supervised) kernel representations are better suited for supervised settings than those using raw input features. This assumption stems from the ability of well-chosen kernel functions to effectively filter out irrelevant features, thereby enhancing the encoder's robustness to highly noisy datasets. To empirically validate this, we conducted a toy experiment using the artificial tree dataset described in Appendix~\ref{sec:artificial_tree}. To simulate a noisy input space, we progressively augmented the dataset with additional features sampled from a uniform distribution $U(0,1)$, corresponding to various signal-to-noise ratios (SNR) $\in \{\infty, 1, 0.1, 0.01, 0.001\}$. We then randomly selected 80\% of each augmented dataset to train both models to regress onto the precomputed training RF-PHATE embeddings. The two MLP regressors shared the exact same architecture and hyperparameters (Appendix~\ref{sec:rfae_exp_setting}), differing only in their input representations. We evaluated the trained models on the remaining 20\% test split and visualized their two-dimensional embeddings under different SNR conditions, along with the ground-truth tree structure and median learning curves over 50 epochs across 10 repetitions, as shown in Fig.~\ref{fig:kernel_vs_feature}.

The training RF-PHATE embeddings accurately capture the underlying ground-truth structure, making them a strong supervisory signal for manifold learning. Our RF-GAP-based encoder proves highly robust to irrelevant features, producing well-structured embeddings even under severe noise conditions (e.g., SNR = 0.001). It consistently converges faster and reaches a better local minimum without overfitting, as evidenced by its test embeddings (middle row), which closely mirror the ground-truth structure. In contrast, the feature-based MLP is much more sensitive to noise, with increasing training loss and disordered embeddings in both training and test sets. Even under low-noise settings (SNR = $\infty$ or 1), it fails to achieve comparable performance, highlighting the superior robustness and generalization ability of our kernel-based encoder.

\begin{figure*}[ht]
    \centering
    \includegraphics[width = 0.9\textwidth]{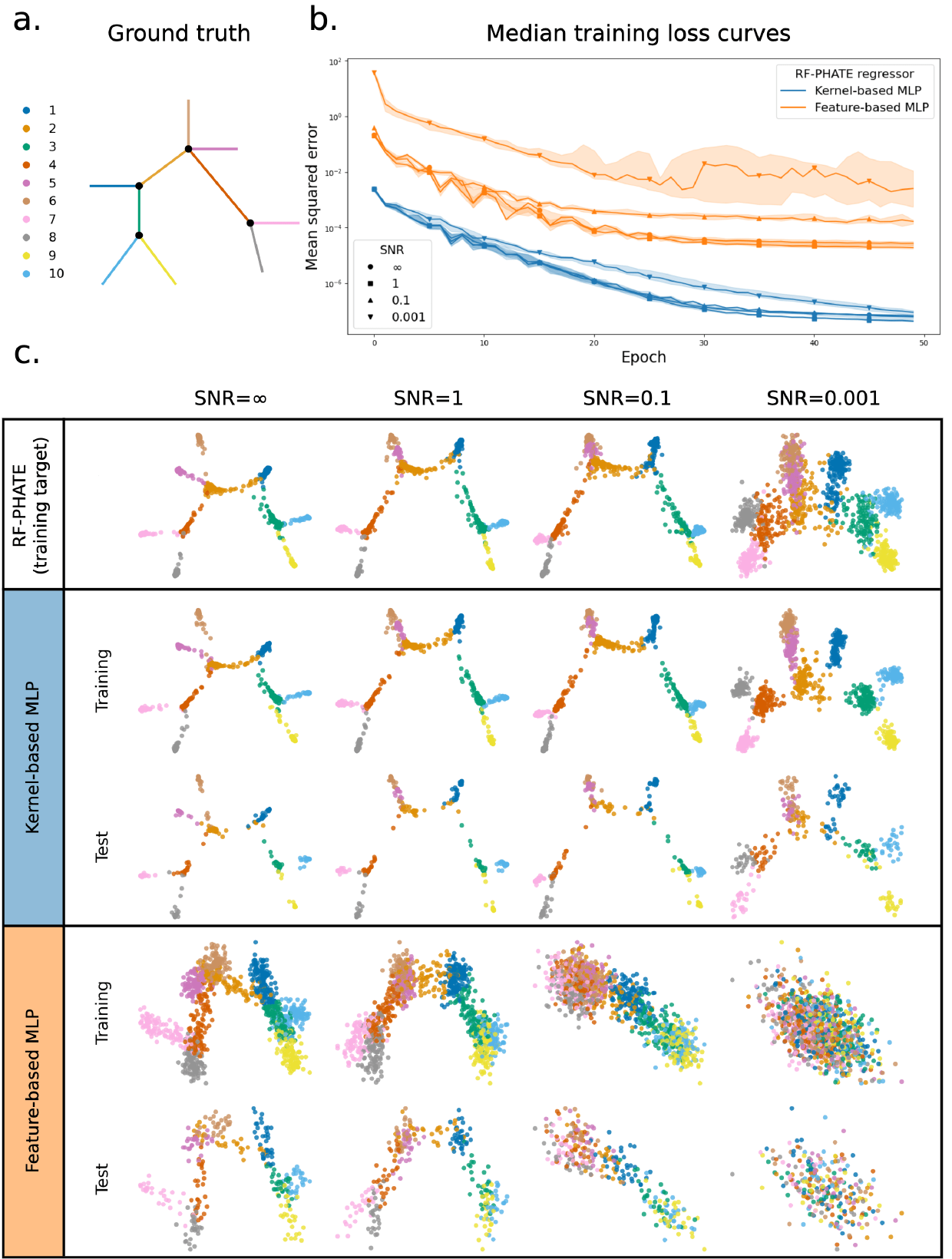}
    \caption{Comparison between the standard feature-based MLP encoder and our proposed RF-GAP kernel-based MLP encoder for regressing onto precomputed RF-PHATE embeddings. \textbf{(a)} Ground-truth tree structure with branch labels (see Appendix~\ref{sec:artificial_tree}). \textbf{(b)} Log-scaled median training MSE with $25^{\text{th}}$ and $75^{\text{th}}$ enclosing percentiles over 50 epochs across 10 repetitions. \textbf{(c)} Training RF-PHATE embeddings (top row), followed by training and test embeddings produced by the RF-GAP-based encoder (middle row) and the feature-based encoder (bottom row) after 50 epochs from a single run. The RF-PHATE embeddings closely match the ground-truth structure and provide a strong target for supervised regression. Our kernel-based encoder remains effective even under high noise levels (e.g., SNR = 0.001), converging more quickly and producing well-structured embeddings with better generalization. In contrast, the feature-based MLP exhibits increasing training loss and disorganized embeddings as noise increases, and often fails to recover meaningful structure even in low-noise settings (SNR = $\infty$, 1), demonstrating the superior robustness of our kernel-based encoders.}
    \label{fig:kernel_vs_feature}
\end{figure*}

\section{Artificial tree construction}\label{sec:artificial_tree}

We constructed the artificial tree data used in Section~\ref{subsec:quant_comp} and Appendix~\ref{sec:kernel_vs_feature} following the method described in the original PHATE paper~\cite{Moon2019phate}. The first branch of the tree consists of 100 linearly spaced points spanning four dimensions, with all other dimensions set to zero. The second branch starts at the endpoint of the first branch, with its 100 points remaining constant in the first four dimensions while progressing linearly in the next four dimensions, leaving all others at zero. Similarly, the third branch progresses linearly in dimensions 9–12, with subsequent branches following the same pattern but differing in length, resulting in 40 dimensions. Each branch endpoint and branching point includes an additional 40 points, and zero-mean Gaussian noise (standard deviation 7) is added to simulate gene expression advancement along the branches. Before visualization, all features are normalized to the range [0, 1].

\section{Feature correlation-aware data perturbation}\label{sec:sampling}

In this section, we detail the procedure for generating perturbed datasets using a correlation-aware random sampling strategy~\cite{kaneko2022cvpfi}. Since ground truth feature importance are rarely available, this approach is employed to generate pseudo-ground truth feature importances as part of our evaluation scheme (Section~\ref{subsec:quantify_oos_embedding}). For each feature \(i\), instead of permuting feature \(i\)'s column values---as in the standard permutation approach---, we reassign them by randomly sampling values from the feature space. Additionally, all other feature column values are randomly replaced with a probability proportional to their absolute correlation with \(i\). In other words, column values for features highly correlated with \(i\) are also replaced by random sampling, while column values for features not correlated with \(i\) remain unchanged. This prevents the determination of fallacious feature importances where all correlated features are assigned zero importance. Refer to Algorithm ~\ref{alg:sampling} for a step-by-step description of this feature-wise data perturbation procedure.

\begin{algorithm}[!htb]
\caption{Feature-wise data perturbation with random sampling}
\label{alg:sampling}

\KwIn{Input data $X$, feature correlation matrix $\mathbf{C}$}
\KwOut{Perturbed datasets $\Tilde{\mathbf{X}}$ for each feature}

Initialize list $\Tilde{\mathbf{X}}$ to store perturbed datasets\;

Generate $\Tilde{X}$ from $X$ by randomly sampling column values without replacement\;

\ForEach{feature $i$}{
    Generate mask matrix $\mathbf{M}$ with elements $\mathbf{M}[i,j] \in \{0,1\}$ sampled from $\text{Bernoulli}(|\mathbf{C}[i,j]|)$\;
    
    Build perturbed dataset: $\Tilde{X}^i = \mathbf{M} \odot \Tilde{X} + (I - \mathbf{M}) \odot X$\;

    Store $\Tilde{\mathbf{X}}[i] = \Tilde{X}^i$\;
}

\Return{$\Tilde{\mathbf{X}}$}

\end{algorithm}

\section{Evaluation metric details for supervised OOS embedding}

\subsection{Structure preservation scores}\label{subsec:def_structure_preservation}

Let $\mathbf{D}_{\mathrm{test}}, \mathbf{D}_{\mathrm{test}}^{\mathrm{emb}} \in \mathbb{R}_{+}^{N_{\mathrm{test}} \times N_{\mathrm{train}}}$ denote the test–train distance matrices in the original and embedded spaces, respectively. We define four structure preservation metrics $s(\mathbf{D}_{\mathrm{test}}, \mathbf{D}_{\mathrm{test}}^{\mathrm{emb}})$, which are used to compute multi-view structural alignment scores introduced in Section~\ref{subsec:quantify_oos_embedding}. We categorize these metrics into local and global types and cite the reference works where they were previously used to assess the quality of embedding methods.

\paragraph{Local Structure Preservation Scores}

\begin{itemize}
  \item \textbf{QNX (Quality of Neighborhood eXtrapolation)}~\cite{sainburg2021parametric_umap, kobak2019heavy}:
  \[
  \textit{QNX}(\mathbf{D}_{\mathrm{test}}, \mathbf{D}_{\mathrm{test}}^{\mathrm{emb}}) := \frac{1}{N_{\mathrm{test}}} \sum_{i=1}^{N_{\mathrm{test}}} \frac{1}{K} \sum_{j \in \mathcal{N}_K^{\mathrm{emb}}(i)} I\left( j \in \mathcal{N}_K^{\mathrm{true}}(i) \right),
  \]
  where $\mathcal{N}_K^{\mathrm{true}}(i)$ are the indices of the $K$ smallest entries in row $\mathbf{D}_{\mathrm{test}}[i, :]$, and $\mathcal{N}_K^{\mathrm{emb}}(i)$ are those in $\mathbf{D}_{\mathrm{test}}^{\mathrm{emb}}[i, :]$.

  \item \textbf{Trustworthiness}~\cite{venna2006local}:
  \[
  \textit{Trust}(\mathbf{D}_{\mathrm{test}}, \mathbf{D}_{\mathrm{test}}^{\mathrm{emb}}) := 1 - \frac{2}{N_{\mathrm{test}} K (2N_{\mathrm{train}} - 3K - 1)} \sum_{i=1}^{N_{\mathrm{test}}} \sum_{j \in \mathcal{U}_i} \left( r_{ij}^{\mathrm{true}} - K \right),
  \]
  where $\mathcal{U}_i = \mathcal{N}_K^{\mathrm{emb}}(i) \setminus \mathcal{N}_K^{\mathrm{true}}(i)$, and $r_{ij}^{\mathrm{true}}$ is the rank of index $j$ in row $\mathbf{D}_{\mathrm{test}}[i, :]$.

\end{itemize}

\paragraph{Global Structure Preservation Scores}

\begin{itemize}
  \item \textbf{Spearman rank correlation}~\cite{Kobak2019tsne}:
  \[
  \textit{Spear}(\mathbf{D}_{\mathrm{test}}, \mathbf{D}_{\mathrm{test}}^{\mathrm{emb}}) := \mathrm{corr}_{\mathrm{rank}}\left( \mathrm{vec}(\mathbf{D}_{\mathrm{test}}),\ \mathrm{vec}(\mathbf{D}_{\mathrm{test}}^{\mathrm{emb}}) \right),
  \]
  where $\mathrm{vec}(\cdot)$ denotes vectorization and $\mathrm{corr}_{\mathrm{rank}}$ is the Spearman rank correlation~\cite{spearman1904general}.

  \item \textbf{Pearson correlation}~\cite{Gildenblat_Pahnke_2025}:
  \[
  \textit{Pearson}(\mathbf{D}_{\mathrm{test}}, \mathbf{D}_{\mathrm{test}}^{\mathrm{emb}}) := \mathrm{corr}\left( \mathrm{vec}(\mathbf{D}_{\mathrm{test}}),\ \mathrm{vec}(\mathbf{D}_{\mathrm{test}}^{\mathrm{emb}}) \right),
  \]
  using the Pearson linear correlation~\cite{pearson1895vii} between flattened test--train distance vectors.

\end{itemize}

For robustness, we averaged local metrics over different neighborhood sizes, ranging from $K=5$ to $K=\sqrt{N_{\mathrm{train}}}$, in steps of 10.

\subsection{Illustration of our structural importance alignment framework}\label{sec:sia_illustration}

Fig.~\ref{fig:sia_illustration} illustrates our SIA framework for evaluating supervised OOS embedding quality using the Sign MNIST (A--K) dataset (Table~\ref{tab:rfae_data}). While both RF-AE and P-TSNE produce locally plausible embeddings, their ability to preserve class-relevant structure differs significantly. RF-AE emphasizes informative regions—such as hand and finger contours—while suppressing irrelevant background variations. In contrast, P-TSNE often attributes high structural importance to background pixels, leading to poorer alignment with classification-relevant features. This discrepancy is reflected in the final local SIA scores: RF-AE achieves a much higher alignment (0.89) than P-TSNE (0.62), confirming that RF-AE better preserves the semantic structure needed for accurate classification in OOS settings. These findings support our qualitative observations from Section~\ref{subsec:oos_viz}.
\begin{figure}[h]
    \centering
    \includegraphics[width = 1\textwidth]{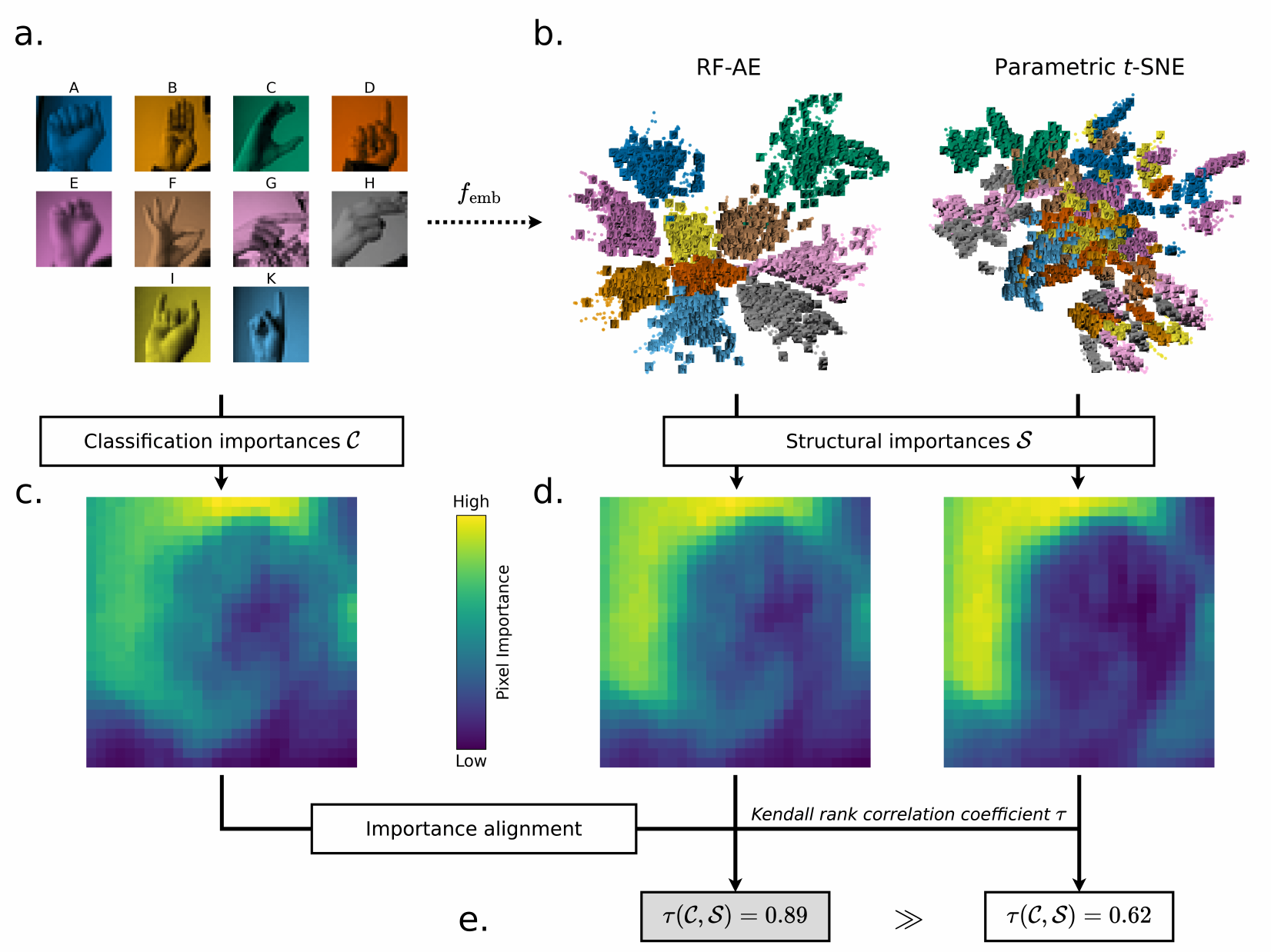}
    \caption{Illustration of the structural importance alignment (SIA) score defined in Section~\ref{subsec:quantify_oos_embedding} for evaluating supervised out-of-sample (OOS) embedding fit. \textbf{a.} Random class samples from the high-dimensional Sign MNIST (A--K) dataset. \textbf{b.} 2D embeddings of training and test (OOS) points from RF-AE (left) and P-TSNE (right), based on a stratified 80/20 random split. Training points are color-coded by label; test points are shown as tinted image thumbnails. \textbf{c.} Pixel-level classification importances from the ensemble baseline classifier (Section~\ref{subsec:quantify_oos_embedding}, Appendix~\ref{sec:baseline_cls_perf}), normalized to $[0,1]$. \textbf{d.} Pixel-level local structure importances ($s=\textit{Trust}$) from OOS RF-AE (left) and P-TSNE (right), also normalized. \textbf{e.} Local SIA scores computed as the Kendall $\tau$ correlation between (c) and (d): RF-AE achieves higher alignment (0.89) than P-TSNE (0.62), suppressing background pixels and focusing on class-relevant regions.}
    \label{fig:sia_illustration}
\end{figure}

\subsection{Baseline classifiers' hyperparameters and performance}\label{sec:baseline_cls_perf}

Since ground-truth classification importances are rarely available, our SIA framework (Section~\ref{subsec:quantify_oos_embedding}) uses a baseline classifier $f_{\mathrm{cls}}$ to derive pseudo-ground-truth importances. To ensure these importances are meaningful, $f_{\mathrm{cls}}$ must achieve reasonably high test accuracy. Table~\ref{tab:acc_summary} reports per-dataset accuracies for both $f_\mathrm{cls}=k$-NN and an ensemble classifier $f_\mathrm{cls} = k$-NN + SVM + MLP combining $k$-NN, SVM, and MLP predictions via equal-weight voting. We use $k = \sqrt{N_{\mathrm{train}}}$ for the $k$-NN classifier. The SVM is implemented using \texttt{scikit-learn}’s \texttt{LinearSVC}~\cite{scikit-learn} with default hyperparameters. The MLP is a two-layer feedforward network with hidden dimensions $h_1 = \lfloor \tfrac{2}{3} \cdot D \rfloor$ and $h_2 = \lfloor \tfrac{1}{3} \cdot D \rfloor$, where $D$ is the input dimensionality. Each hidden layer is followed by ReLU activation, dropout (rate 0.2), and layer normalization. Weight normalization is applied to the first two linear layers. The final layer is a standard linear projection without activation.

We find that the ensemble consistently improves upon standalone $k$-NN and achieves above 60\% accuracy on all datasets, making it a suitable proxy for generating classification importances. Nonetheless, $k$-NN alone performs reasonably well, falling below 60\% accuracy on OrganC MNIST dataset. For a detailed comparison of SIA scores using $k$-NN instead of the ensemble, see Section~\ref{sec:sia_varying_imp}.

\begin{table*}[ht]
\small
\caption{Average test classification accuracy (mean $\pm$ std) per dataset (see Appendix~\ref{sec:data}), using a single $k$-NN classifier (left column) and an ensemble of $k$-NN, linear SVM, and MLP classifiers (right column). The ensemble generally outperforms the standalone $k$-NN, making it a robust reference for generating classification feature importances.}
\label{tab:acc_summary}
\vskip 0.15in
\centering
\begin{tabular}{lcc}
\toprule
\textsc{Dataset} & \textsc{$k$-NN} & \textsc{$k$-NN + SVM + MLP} \\
\midrule
\textsc{QSAR Biodegradation} & 0.835 $\pm$ 0.032 & 0.854 $\pm$ 0.028 \\
\textsc{Blood MNIST} & 0.742 $\pm$ 0.000 & 0.761 $\pm$ 0.049 \\
\textsc{Cardiotocography} & 0.662 $\pm$ 0.026 & 0.681 $\pm$ 0.023 \\
\textsc{Chess} & 0.914 $\pm$ 0.012 & 0.942 $\pm$ 0.012 \\
\textsc{Diabetic Retinopathy Debrecen} & 0.657 $\pm$ 0.043 & 0.686 $\pm$ 0.031 \\
\textsc{Fashion MNIST (test)} & 0.777 $\pm$ 0.007 & 0.827 $\pm$ 0.007 \\
\textsc{GTZAN (3-sec)} & 0.708 $\pm$ 0.006 & 0.678 $\pm$ 0.013 \\
\textsc{HAR (Using Smartphones)} & 0.887 $\pm$ 0.000 & 0.917 $\pm$ 0.009 \\
\textsc{Isolet} & 0.906 $\pm$ 0.000 & 0.931 $\pm$ 0.004 \\
\textsc{Landsat Satellite} & 0.858 $\pm$ 0.000 & 0.829 $\pm$ 0.010 \\
\textsc{MNIST (test)} & 0.894 $\pm$ 0.008 & 0.922 $\pm$ 0.006 \\
\textsc{Obesity} & 0.625 $\pm$ 0.018 & 0.661 $\pm$ 0.026 \\
\textsc{Optical Burst Switching Network} & 0.745 $\pm$ 0.028 & 0.743 $\pm$ 0.023 \\
\textsc{Optical Digits} & 0.953 $\pm$ 0.000 & 0.948 $\pm$ 0.004 \\
\textsc{OrganC MNIST} & 0.473 $\pm$ 0.000 & 0.627 $\pm$ 0.004 \\
\textsc{Sign MNIST (A--K)} & 0.908 $\pm$ 0.007 & 0.940 $\pm$ 0.008 \\
\textsc{Spambase} & 0.860 $\pm$ 0.008 & 0.875 $\pm$ 0.008 \\
\textsc{Sports Articles} & 0.809 $\pm$ 0.019 & 0.818 $\pm$ 0.019 \\
\textsc{USPS} & 0.871 $\pm$ 0.000 & 0.891 $\pm$ 0.002 \\
\textsc{Waveform} & 0.848 $\pm$ 0.012 & 0.860 $\pm$ 0.014 \\
\bottomrule
\end{tabular}
\end{table*}

\section{Description of the datasets}\label{sec:data}

Table~\ref{tab:rfae_data} provides additional details on the datasets used for the quantitative and qualitative comparisons between RF-AE and other methods. Sign MNIST (A–K)~\cite{sign_language_mnist}, MNIST (test subset)~\cite{lecun2010mnist}, Fashion MNIST (test subset)~\cite{xiao2017fashionmnist}, GTZAN (3-second version)~\cite{sturm2013gtzan} and USPS~\cite{USPS} were obtained from \href{https://www.kaggle.com}{Kaggle}. The Sign MNIST (A–K) dataset is a subset of the original, containing the first 10 letters (excluding J, which requires motion). Blood MNIST and OrganC MNIST (Med MNIST family~\cite{medmnistv1, medmnistv2}) were obtained from \href{https://zenodo.org/records/10519652}{Zenodo}. 
All other datasets are publicly available from the \href{https://archive.ics.uci.edu/} {UCI Machine Learning Repository}.

\begin{table*}[ht]
\small
\caption{Description of the 20 datasets used in our experiments, grouped by data modality.}
\label{tab:rfae_data}
\vskip 0.15in
\centering
\begin{tabular}{lcccc}
\toprule
\textsc{Dataset} & \textsc{Size} & \textsc{Test \%} & \textsc{Dimensions} & \textsc{Classes} \\
\midrule
\multicolumn{5}{l}{\textsc{Tabular / Clinical}} \\
\quad\textsc{Cardiotocography} & 2126 & 0.20 & 21 & 10 \\
\quad\textsc{Diabetic Retinopathy Debrecen} & 1151 & 0.20 & 19 & 2 \\
\quad\textsc{Obesity} & 2111 & 0.20 & 16 & 7 \\
\quad\textsc{QSAR Biodegradation} & 1055 & 0.20 & 41 & 2 \\[0.8ex]
\multicolumn{5}{l}{\textsc{Text / NLP}} \\
\quad\textsc{Spambase} & 4601 & 0.20 & 57 & 2 \\
\quad\textsc{Sports Articles} & 1000 & 0.20 & 59 & 2 \\[0.8ex]
\multicolumn{5}{l}{\textsc{Sensor / Time Series}} \\
\quad\textsc{HAR (Using Smartphones)} & 10299 & 0.29 & 561 & 6 \\
\quad\textsc{Isolet} & 7797 & 0.20 & 617 & 26 \\
\quad\textsc{Waveform} & 5000 & 0.20 & 40 & 3 \\
\quad\textsc{Landsat Satellite} & 6435 & 0.31 & 36 & 6 \\[0.8ex]
\multicolumn{5}{l}{\textsc{Image (General)}} \\
\quad\textsc{Optical Digits} & 5620 & 0.32 & 64 & 10 \\
\quad\textsc{USPS} & 9298 & 0.22 & 256 & 10 \\
\quad\textsc{MNIST (test)} & 10000 & 0.20 & 784 & 10 \\
\quad\textsc{Fashion MNIST (test)} & 10000 & 0.20 & 784 & 10 \\
\quad\textsc{Sign MNIST (A--K)} & 14482 & 0.20 & 784 & 10 \\[0.8ex]
\multicolumn{5}{l}{\textsc{Image (Biomedical)}} \\
\quad\textsc{Blood MNIST} & 15380 & 0.22 & 2352 & 8 \\
\quad\textsc{OrganC MNIST} & 21191 & 0.39 & 784 & 11 \\[0.8ex]
\multicolumn{5}{l}{\textsc{Audio}} \\
\quad\textsc{GTZAN (3-sec)} & 9990 & 0.20 & 57 & 10 \\[0.8ex]
\multicolumn{5}{l}{\textsc{Network / Traffic}} \\
\quad\textsc{Optical Burst Switching Network} & 1060 & 0.20 & 21 & 4 \\[0.8ex]
\multicolumn{5}{l}{\textsc{Games / Logic}} \\
\quad\textsc{Chess} & 3196 & 0.20 & 36 & 2 \\[0.8ex]
\bottomrule
\end{tabular}
\end{table*}

\section{Experimental setting}\label{sec:rfae_exp_setting}

\subsection{Model implementations and hyperparameters}\label{subsec:exp_setting}
Unless otherwise specified, all methods were run with their default hyperparameters in our experiments.

\begin{itemize}
\item \textbf{RF-AE:} Implemented in PyTorch. The encoder \( f \) consisted of three hidden layers with sizes 800, 400, and 100. The bottleneck layer was set to dimension 2 for visualization. The decoder \( g \) was symmetric with layers of sizes 100, 400, and 800, followed by an output layer matching the input dimensionality. ELU activations were used throughout, except for the bottleneck (identity) and output (softmax) layers. Training was performed using the AdamW optimizer~\cite{loshchilov2018decoupled} with a learning rate of \(10^{-3}\), batch size of 512, weight decay of \(10^{-5}\), and 200 epochs without early stopping. We set the default $\lambda$ and $N^*$ to $0.01$ and $0.1N_{\mathrm{train}}$, respectively.

\item \textbf{SSNP, CE, and vanilla AE:} Implemented using the same architecture and activations as RF-AE. For SSNP, we followed the authors’ recommendations: a sigmoid output activation and a reconstruction-classification loss balance of 0.5. For CE and vanilla AE, the output activation was the identity function.

\item \textbf{Parametric $t$-SNE and UMAP:} Implemented following Damrich et al.~\cite{damrich2023from}, using the InfoNCE loss~\cite{oord2018representation}; available at \url{https://github.com/sdamrich/cl-tsne-umap}.

\item \textbf{Parametric supervised UMAP:} Official implementation from \url{https://github.com/lmcinnes/umap}.

\item \textbf{PaCMAP:} From \url{https://github.com/YingfanWang/PaCMAP}.

\item \textbf{CEBRA:} From \url{https://github.com/AdaptiveMotorControlLab/CEBRA}. We used 200 training epochs and a batch size of 512, as recommended by the authors.

\item \textbf{SPCA:} From \url{https://github.com/bghojogh/Principal-Component-Analysis}.

\item \textbf{PCA, NCA, and PLS-DA:} Implemented using the \texttt{scikit-learn} library~\cite{scikit-learn}.
\end{itemize}


\subsection{Compute resources}
\label{subsec:exp-compute-resource}
Experiments were conducted on a shared computing environment with access to both GPU and CPU resources. For models requiring GPU acceleration, we used:
\begin{itemize}
\item 1 GPU with at least 40 GB of memory (e.g., NVIDIA A100 40GB, H100 80GB, or equivalent),
\item 1 CPU with 128 GB of RAM.
\end{itemize}
For models that do not require GPU acceleration, computations were performed using CPU only, with a minimum of 128 GB of RAM.

We conducted experiments across 20 datasets for our RF-AE model and 13 baseline methods, using multiple random seeds to report the mean and standard deviation of performance metrics. All hyperparameters and configurations were managed using Hydra~\cite{Yadan2019Hydra}. Code and configuration files will be released to ensure full reproducibility.

The runtime for RF-AE training and the entire evaluation process for individual experiments, where each experiment is defined as running one model on one dataset with a single random seed, ranged from 1 to 6 hours depending on the dataset size.

\section{Ablation experiments}\label{sec:ablation_supplemental}
We performed ablation experiments on the two main RF-AE hyperparameters: the loss balancing parameter $\lambda$ and the number of selected landmarks $N^*$. We report local/global SIA scores and $k$-NN accuracies across combinations $(\lambda, N^*) \in \{1, 0.1, 0.01, 0.001, 0\} \times \{pN_{\mathrm{train}} \mid p = 0.02, 0.05, 0.1, 0.2, 1\}$ in Table~\ref{tab:rfae_ablation_corrected}. Surprisingly, reducing the number of selected prototypes leads to overall improvements in both $k$-NN accuracy while preserving SIA. We hypothesize that this may be attributed to the reduced input dimensionality of the RF-AE network, which effectively lowers its complexity and introduces additional implicit regularization. Furthermore, selecting only the most representative instances per class may help denoise the training process, thereby enhancing the model’s ability to preserve class-relevant features in the embedding space.

For the loss balancing hyperparameter, setting $\lambda = 1$ (i.e., an unconstrained RF-AE) yields relatively high accuracy but results in a substantial decline in supervised structure preservation. This is expected, as unconstrained autoencoders have been shown to poorly capture the underlying data geometry~\cite{duque2022geometry, Nazari2023geometric}. On the other hand, $\lambda = 0$, which corresponds to the RF-PHATE kernel-based MLP, leads to both lower accuracy and diminished global SIA, offering no improvement over the standard RF-PHATE extension results reported in Table~\ref{tab:quantitative_comp}.

Across a broad range of hyperparameters—specifically, $\lambda \in \{0.1, 0.01, 0.001\}$ and $N^* \in \{pN_{\mathrm{train}} \mid p = 0.02, 0.05, 0.1, 0.2, 1\}$—RF-AE consistently ranks among the top 3 methods across all metrics in Table~\ref{tab:quantitative_comp}, highlighting its strong robustness to hyperparameter choices.

\begin{table}[ht]
\caption{Local ($s=\textit{QNX},\textit{Trust}$) and global ($s=\textit{Spear},\textit{Pearson}$) SIA scores, and test $k$-NN accuracy for RF-AE variants across values of $\lambda$ and $N^*$. Scores are shown as mean ± std across 20
datasets and 10 repetitions. Each score is compared with baseline models in Table~\ref{tab:quantitative_comp}, and highlighted only if it ranks among the top three overall. Top three values per metric are highlighted using underlined bold italic (first), bold italic (second), and italic (third). RF-AE demonstrates strong robustness for $\lambda \in \{0.1, 0.01, 0.001\}$ across varying prototype count, consistently ranking among the top 3 methods. Fewer prototypes improve $k$-NN accuracy while preserving SIA, likely due to implicit regularization and class-level denoising. In contrast, extreme $\lambda$ values lead to degraded SIA ($\lambda = 1$) or both SIA and accuracy ($\lambda = 0$)}
\label{tab:rfae_ablation_corrected}
\centering
\small
\begin{sc}
\begin{tabular}{lccccc}
\toprule
 & \multicolumn{2}{c}{Local SIA} & \multicolumn{2}{c}{Global SIA} &  \\
\cmidrule(r){2-5}
 & QNX & Trust & Spear & Pearson & $k$-NN Acc \\
\midrule
\multicolumn{6}{l}{\textbf{\boldmath$\lambda=0$}} \\
$N^* = 2\%$ & \textbf{\textit{0.805 ± 0.025}} & \textbf{\textit{0.819 ± 0.027}} & 0.746 ± 0.041 & \textit{0.750 ± 0.042} & \underline{\textbf{\textit{0.810 ± 0.010}}} \\
$N^* = 5\%$ & \textbf{\textit{0.804 ± 0.026}} & \textbf{\textit{0.818 ± 0.025}} & \textit{0.749 ± 0.038} & \textit{0.750 ± 0.042} & \textbf{\textit{0.808 ± 0.010}} \\
$N^* = 10\%$ & \textbf{\textit{0.803 ± 0.027}} & \textit{0.817 ± 0.025} & 0.747 ± 0.038 & 0.749 ± 0.042 & \textbf{\textit{0.806 ± 0.010}} \\
$N^* = 20\%$ & \textbf{\textit{0.805 ± 0.026}} & \textbf{\textit{0.818 ± 0.025}} & 0.747 ± 0.039 & 0.749 ± 0.042 & \textbf{\textit{0.798 ± 0.010}} \\
$N^* = 100\%$ & \textit{0.802 ± 0.026} & \textit{0.817 ± 0.025} & \textit{0.749 ± 0.040} & 0.749 ± 0.043 & 0.792 ± 0.010 \\
\midrule
\multicolumn{6}{l}{\textbf{\boldmath$\lambda=0.001$}} \\
$N^* = 2\%$ & \textbf{\textit{0.802 ± 0.026}} & \textbf{\textit{0.820 ± 0.026}} & \textit{0.759 ± 0.042} & \textit{0.755 ± 0.046} & \underline{\textbf{\textit{0.830 ± 0.010}}} \\
$N^* = 5\%$ & \textbf{\textit{0.800 ± 0.024}} & \textit{0.817 ± 0.026} & \textit{0.764 ± 0.040} & \textit{0.757 ± 0.043} & \underline{\textbf{\textit{0.832 ± 0.009}}} \\
$N^* = 10\%$ & \textbf{\textit{0.801 ± 0.025}} & \textbf{\textit{0.819 ± 0.026}} & \textit{0.763 ± 0.040} & \textbf{\textit{0.760 ± 0.044}} & \underline{\textbf{\textit{0.827 ± 0.010}}} \\
$N^* = 20\%$ & \textbf{\textit{0.799 ± 0.026}} & \textbf{\textit{0.820 ± 0.026}} & \textit{0.760 ± 0.043} & \textit{0.755 ± 0.046} & \underline{\textbf{\textit{0.825 ± 0.009}}} \\
$N^* = 100\%$ & \textit{0.797 ± 0.028} & \textbf{\textit{0.819 ± 0.027}} & \textit{0.757 ± 0.042} & \textit{0.757 ± 0.044} & \textbf{\textit{0.802 ± 0.013}} \\
\midrule
\multicolumn{6}{l}{\textbf{\boldmath$\lambda=0.01$}} \\
$N^* = 2\%$ & \textbf{\textit{0.801 ± 0.026}} & \textbf{\textit{0.819 ± 0.025}} & \underline{\textbf{\textit{0.772 ± 0.043}}} & \underline{\textbf{\textit{0.774 ± 0.044}}} & \underline{\textbf{\textit{0.833 ± 0.012}}} \\
$N^* = 5\%$ & \textbf{\textit{0.801 ± 0.025}} & \textbf{\textit{0.818 ± 0.024}} & \underline{\textbf{\textit{0.775 ± 0.045}}} & \underline{\textbf{\textit{0.775 ± 0.047}}} & \underline{\textbf{\textit{0.836 ± 0.011}}} \\
$N^* = 10\%$ & \textbf{\textit{0.800 ± 0.025}} & \textbf{\textit{0.818 ± 0.025}} & \underline{\textbf{\textit{0.776 ± 0.051}}} & \underline{\textbf{\textit{0.776 ± 0.050}}} & \underline{\textbf{\textit{0.831 ± 0.011}}} \\
$N^* = 20\%$ & \textbf{\textit{0.800 ± 0.026}} & \textbf{\textit{0.818 ± 0.026}} & \underline{\textbf{\textit{0.778 ± 0.039}}} & \underline{\textbf{\textit{0.776 ± 0.039}}} & \underline{\textbf{\textit{0.830 ± 0.011}}} \\
$N^* = 100\%$ & 0.794 ± 0.025 & \textit{0.816 ± 0.026} & \textbf{\textit{0.766 ± 0.046}} & \underline{\textbf{\textit{0.769 ± 0.043}}} & \underline{\textbf{\textit{0.809 ± 0.013}}} \\
\midrule
\multicolumn{6}{l}{\textbf{\boldmath$\lambda=0.1$}} \\
$N^* = 2\%$ & \textbf{\textit{0.800 ± 0.027}} & \textbf{\textit{0.819 ± 0.024}} & \textbf{\textit{0.767 ± 0.060}} & \underline{\textbf{\textit{0.780 ± 0.050}}} & \underline{\textbf{\textit{0.834 ± 0.010}}} \\
$N^* = 5\%$ & \textbf{\textit{0.800 ± 0.025}} & \textbf{\textit{0.818 ± 0.025}} & \underline{\textbf{\textit{0.776 ± 0.062}}} & \textit{0.783 ± 0.055} & \underline{\textbf{\textit{0.836 ± 0.012}}} \\
$N^* = 10\%$ & \textbf{\textit{0.798 ± 0.026}} & \textit{0.817 ± 0.025} & \underline{\textbf{\textit{0.785 ± 0.045}}} & \underline{\textbf{\textit{0.788 ± 0.045}}} & \underline{\textbf{\textit{0.831 ± 0.011}}} \\
$N^* = 20\%$ & \textbf{\textit{0.800 ± 0.025}} & \textbf{\textit{0.818 ± 0.024}} & \underline{\textbf{\textit{0.784 ± 0.043}}} & \textbf{\textit{0.786 ± 0.044}} & \underline{\textbf{\textit{0.831 ± 0.010}}} \\
$N^* = 100\%$ & 0.789 ± 0.028 & \textit{0.813 ± 0.029} & \underline{\textbf{\textit{0.779 ± 0.057}}} & \textit{0.783 ± 0.054} & \textbf{\textit{0.807 ± 0.015}} \\
\midrule
\multicolumn{6}{l}{\textbf{\boldmath$\lambda=1$}} \\
$N^* = 2\%$ & \textbf{\textit{0.802 ± 0.026}} & \textbf{\textit{0.818 ± 0.026}} & 0.669 ± 0.112 & 0.672 ± 0.125 & \underline{\textbf{\textit{0.838 ± 0.009}}} \\
$N^* = 5\%$ & \textbf{\textit{0.800 ± 0.027}} & \textit{0.815 ± 0.024} & 0.681 ± 0.117 & 0.680 ± 0.115 & \underline{\textbf{\textit{0.839 ± 0.010}}} \\
$N^* = 10\%$ & \textbf{\textit{0.799 ± 0.026}} & \textit{0.816 ± 0.026} & 0.701 ± 0.105 & 0.700 ± 0.108 & \underline{\textbf{\textit{0.837 ± 0.011}}} \\
$N^* = 20\%$ & \textbf{\textit{0.799 ± 0.025}} & \textit{0.814 ± 0.024} & 0.698 ± 0.095 & 0.696 ± 0.105 & \underline{\textbf{\textit{0.831 ± 0.011}}} \\
$N^* = 100\%$ & 0.793 ± 0.028 & \textit{0.808 ± 0.029} & 0.692 ± 0.078 & 0.696 ± 0.068 & 0.772 ± 0.021 \\
\bottomrule
\end{tabular}
\end{sc}
\end{table}

Fig.~\ref{fig:sign_ablation_plot} visually demonstrates the impact of varying $(\lambda, N^*) \in \{1, 0.1,0.01,0.001,0\} \times \{ pN_{\mathrm{train}}\mid p=0.02,0.05,0.1,0.2,1\}$ on Sign MNIST (A--K). Visually, the main source of variation is the loss balancing hyperparameter $\lambda$. The number of selected prototypes has negligeable impact. When RF-AE is unconstrained ($\lambda = 1$, first column), the resulting embeddings appear more distorted and less structured. This is consistent with recent findings showing that unregularized autoencoders often fail to produce human-interpretable visualizations that preserve the intrinsic geometry of the data~\cite{duque2022geometry, Nazari2023geometric}. On the contrary, full geometric constraint ($\lambda = 0$, last column) simply replicates the RF-PHATE embedding, without clear qualitative benefits compared to the default linear kernel extension (Fig.~\ref{fig:viz_main}). To effectively balance reconstruction and geometric losses, the optimal range for $\lambda$ lies approximately between 0.001—yielding branching structures akin to RF-PHATE but with more pronounced separation—and 0.1, which produces more compact, globular embeddings with enhanced class separability. A similar qualitative assessment can be made for OrganC MNIST in Fig.~\ref{fig:organ_ablation_plot}.

\begin{figure*}[t]
    \centering
    \includegraphics[width = 1\textwidth]{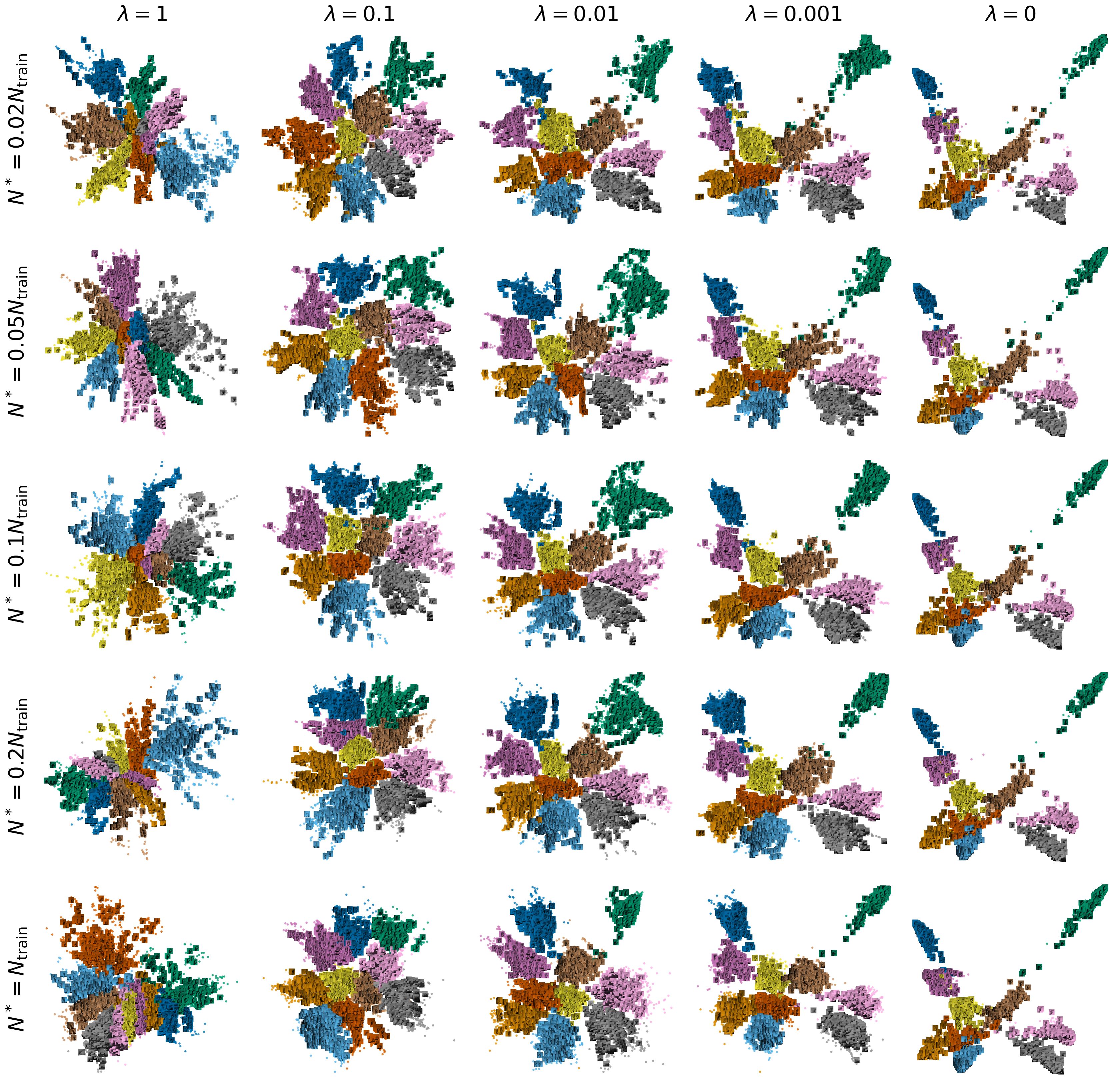}
    \caption{RF-AE training and test embeddings on Sign MNIST (A--K) for various $(\lambda, N^*)$ configurations, where $\lambda$ decreases column-wise from 1 (unconstrained RF-AE) to 0 (RF-PHATE kernel-based MLP extension), and $N^*$ increases row-wise from 2\% to 100\% of the training set size.}
    \label{fig:sign_ablation_plot}
\end{figure*}

\begin{figure*}[t]
    \centering
    \includegraphics[width = 1\textwidth]{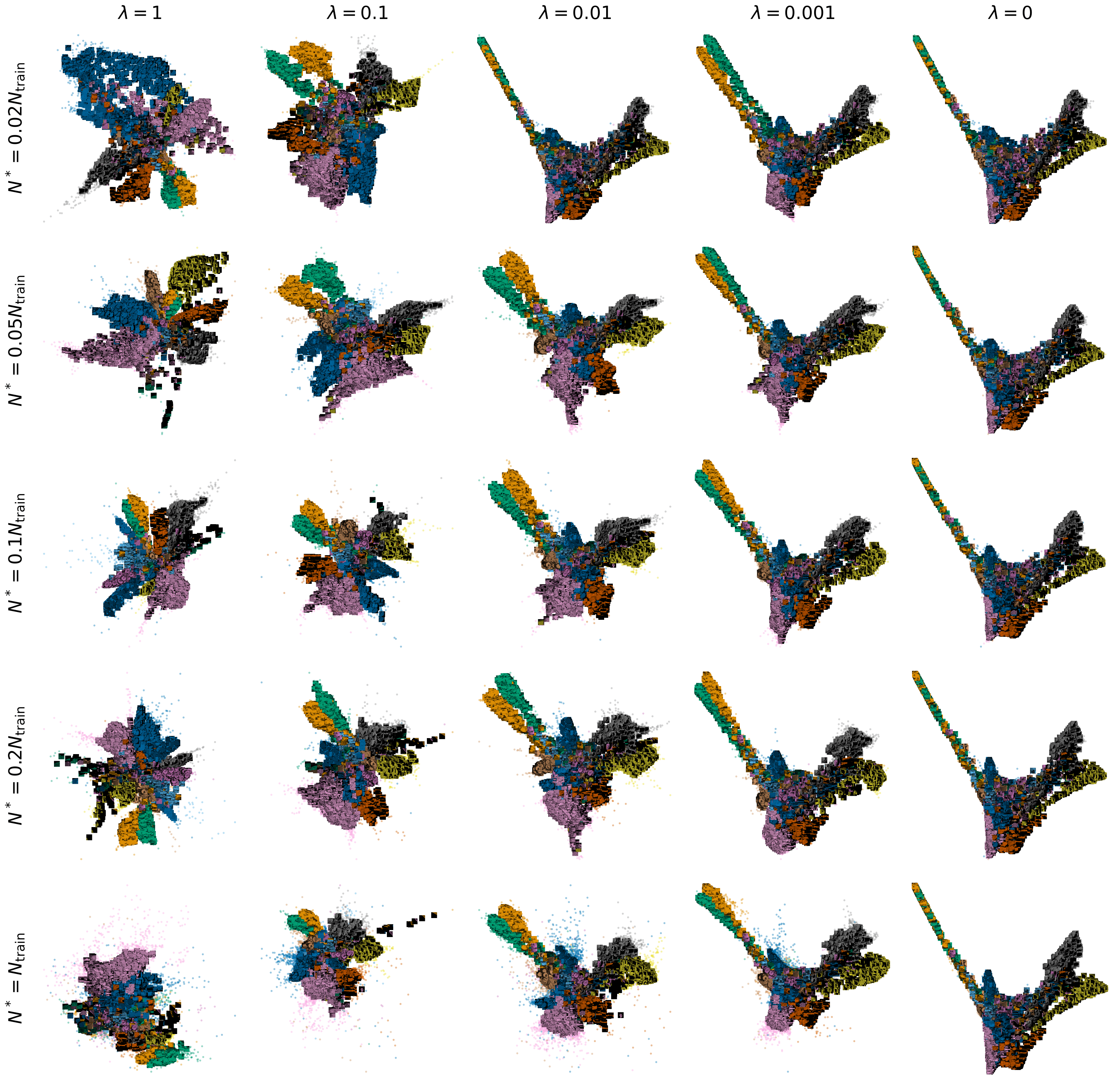}
    \caption{RF-AE training and test embeddings on OrganC MNIST for various $(\lambda, N^*)$ configurations, where $\lambda$ decreases column-wise from 1 (unconstrained RF-AE) to 0 (RF-PHATE kernel-based MLP extension), and $N^*$ increases row-wise from 2\% to 100\% of the training set size.}
    \label{fig:organ_ablation_plot}
\end{figure*}

From these results, we draw two practical guidelines to help users select suitable hyperparameters for their specific application:
\begin{itemize}
    \item \textbf{Prototype selection $N^*$:} Selecting as few as 2\% of training points as prototypes is a good starting point to preserve supervised structure while maximizing class separability. If minimizing training and inference time is essential, users may further reduce the number of selected prototypes to accelerate computation.
    \item \textbf{Loss balancing parameter $\lambda$:} Values of $\lambda$ in the range $[0.001, 0.1]$ yield comparable embedding quality but differ in qualitative behavior. Lower values (e.g., $\lambda \approx 0.001$) produce branching structures similar to RF-PHATE, enhancing interpretability of within-class transitions. Higher values (e.g., $\lambda \approx 0.1$) shift the focus toward class separability, resulting in more globular, well-separated clusters. We recommend $\lambda \approx 0.001$ when transitional or trajectory-like structure is expected, and $\lambda \approx 0.1$ when class boundaries are the primary concern.
\end{itemize}

We also report the training and test time of RF-AE models with varying prototype percentages on the OrganC MNIST dataset (subset), as shown in Figure~\ref{fig:time}. Using a smaller percentage of prototypes significantly reduces computation time, achieving up to a 2 times speedup in training compared to using all data points. For extending embeddings to OOS data points, prototype selection reduces test time by up to 100 times, highlighting its effectiveness in accelerating inference. This demonstrates the computational advantage of prototype selection in the RF-AE framework, particularly for improving test-time efficiency in OOS extension. However, RF-AE still faces scalability challenges, as generating RF-GAP proximity matrices remains computationally intensive. As the dataset size grows, the overall training process can become increasingly time-consuming.

\begin{figure}[!htb]
    \centering
    \includegraphics[width = \textwidth]{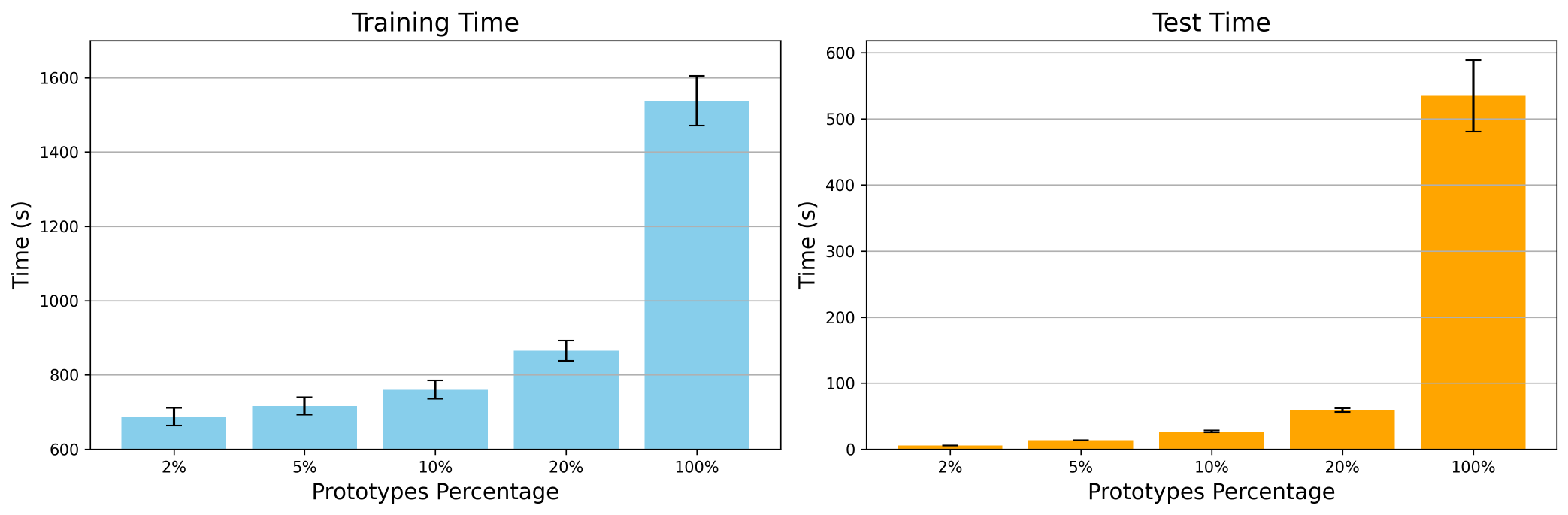}
    \caption{Computation time for OrganC MNIST dataset of RF-AE models with varying prototype percentages.}
    \label{fig:time}
\end{figure}


\section{SIA performance comparison under different classification importance strategies}\label{sec:sia_varying_imp}

To show that RF-AE’s performance is not dependent on the choice classification importances $\mathcal{C}_i$ (Section~\ref{subsec:quantify_oos_embedding}), we repeated the quantitative analysis from Section~\ref{subsec:quant_comp} using two alternative strategies:
\begin{itemize}
\item \textbf{$k$-NN strategy:} We replaced our baseline ensemble classifier with a standalone $k$-NN model.
\item \textbf{Aggregate strategy:} Instead of deriving feature importances from the ensemble’s accuracy drop, we computed importances independently using each of the three classifiers—$k$-NN, SVM, and MLP—resulting in the following sets:
\begin{align*}
\mathcal{C}^{\text{$k$-NN}} &= \{\mathcal{C}_i^{\text{$k$-NN}} \mid i = 1, \ldots, D\},\\
\mathcal{C}^{\text{SVM}} &= \{\mathcal{C}_i^{\text{SVM}} \mid i = 1, \ldots, D\},\\
\mathcal{C}^{\text{MLP}} &= \{\mathcal{C}_i^{\text{MLP}} \mid i = 1, \ldots, D\}.
\end{align*}
We then averaged these to obtain an aggregated importance set:
\[
\mathcal{C}^{\text{agg}} = \frac{1}{3} \left( \mathcal{C}^{\text{$k$-NN}} + \mathcal{C}^{\text{SVM}} + \mathcal{C}^{\text{MLP}} \right).
\]
\end{itemize}

Table~\ref{tab:sia_varying_imp} reports local and global SIA scores for RF-AE and 13 baseline methods using our proposed ensemble classifier (Sections~\ref{subsec:quantify_oos_embedding} and~\ref{subsec:quant_comp}) as well as the two alternative importance strategies. Overall, RF-AE consistently ranks among the top three methods across all metrics, regardless of the chosen importance strategy. This suggests that RF-AE more effectively preserves the underlying important structure, making it more likely to reflect meaningful feature hierarchies in its embeddings compared to other baselines.

\begin{table}[ht]
\caption{Local ($s=\textit{QNX}, \textit{Trust}$) and global ($s=\textit{Spear}, \textit{Pearson}$) SIA scores for RF-AE and 13 baseline methods, computed using two strategies: $k$-NN-based classification importances (top) and aggregated importances averaged over $k$-NN, SVM, and MLP classifiers (bottom). Scores are reported as mean $\pm$ standard deviation across 20 datasets and 10 repetitions. In general, RF-AE outperforms other models in both local and global SIA, regardless of the importance strategy. Top three values in each column are highlighted with underlined bold italics (first), bold italics (second), and italics (third). Supervised methods are marked by an asterisk.}
\label{tab:sia_varying_imp}
\centering
\small
\begin{sc}
\begin{tabular}{lcccc}
\toprule
 & \multicolumn{2}{c}{Local SIA} & \multicolumn{2}{c}{Global SIA} \\
 \cmidrule(r){2-5}
 & QNX & Trust & Spear & Pearson \\
\midrule
\multicolumn{5}{c}{Ensemble Importances} \\
RF-AE* & \textbf{\textit{0.800 ± 0.025}} & \textbf{\textit{0.818 ± 0.025}} & \underline{\textbf{\textit{0.776 ± 0.051}}} & \underline{\textbf{\textit{0.776 ± 0.050}}} \\
SSNP* & 0.760 ± 0.047 & 0.772 ± 0.045 & 0.685 ± 0.089 & 0.694 ± 0.080 \\
P-SUMAP* & 0.756 ± 0.028 & 0.768 ± 0.025 & 0.647 ± 0.048 & 0.647 ± 0.048 \\
CE* & 0.795 ± 0.050 & \textit{0.818 ± 0.051} & \textit{0.765 ± 0.051} & \textbf{\textit{0.763 ± 0.054}} \\
RF-PHATE* & \textit{0.798 ± 0.027} & \underline{\textbf{\textit{0.823 ± 0.025}}} & 0.749 ± 0.040 & 0.750 ± 0.043 \\
NCA* & \underline{\textbf{\textit{0.808 ± 0.027}}} & 0.805 ± 0.025 & \textbf{\textit{0.771 ± 0.032}} & \textit{0.759 ± 0.033} \\
PACMAP & 0.749 ± 0.026 & 0.758 ± 0.025 & 0.688 ± 0.029 & 0.688 ± 0.029 \\
P-TSNE & 0.743 ± 0.028 & 0.747 ± 0.028 & 0.684 ± 0.036 & 0.666 ± 0.038 \\
AE & 0.744 ± 0.027 & 0.751 ± 0.029 & 0.695 ± 0.044 & 0.655 ± 0.053 \\
P-UMAP & 0.757 ± 0.027 & 0.744 ± 0.028 & 0.674 ± 0.035 & 0.657 ± 0.038 \\
SPCA* & 0.767 ± 0.026 & 0.759 ± 0.030 & 0.741 ± 0.031 & 0.738 ± 0.032 \\
PLS-DA* & 0.715 ± 0.026 & 0.708 ± 0.028 & 0.659 ± 0.027 & 0.639 ± 0.028 \\
CEBRA* & 0.780 ± 0.045 & 0.775 ± 0.050 & 0.735 ± 0.062 & 0.728 ± 0.068 \\
PCA & 0.745 ± 0.027 & 0.742 ± 0.026 & 0.733 ± 0.027 & 0.727 ± 0.028 \\
\\
\multicolumn{5}{c}{Standalone $k$-NN Importances} \\
RF-AE* & \textit{0.826 ± 0.025} & \textbf{\textit{0.830 ± 0.025}} & \underline{\textbf{\textit{0.780 ± 0.046}}} & \underline{\textbf{\textit{0.787 ± 0.045}}} \\
SSNP* & 0.780 ± 0.050 & 0.779 ± 0.046 & 0.681 ± 0.094 & 0.690 ± 0.084 \\
P-SUMAP* & 0.780 ± 0.026 & 0.788 ± 0.023 & 0.666 ± 0.049 & 0.666 ± 0.049 \\
CE* & \textbf{\textit{0.829 ± 0.050}} & \textit{0.821 ± 0.048} & \textit{0.763 ± 0.051} & \textit{0.760 ± 0.050} \\
RF-PHATE* & \underline{\textbf{\textit{0.834 ± 0.025}}} & \underline{\textbf{\textit{0.835 ± 0.024}}} & 0.753 ± 0.034 & \textit{0.760 ± 0.038} \\
NCA* & \textit{0.826 ± 0.022} & 0.811 ± 0.025 & \textbf{\textit{0.774 ± 0.032}} & \textbf{\textit{0.761 ± 0.031}} \\
PACMAP & 0.771 ± 0.023 & 0.777 ± 0.022 & 0.708 ± 0.027 & 0.711 ± 0.028 \\
P-TSNE & 0.766 ± 0.025 & 0.767 ± 0.025 & 0.702 ± 0.030 & 0.683 ± 0.035 \\
AE & 0.762 ± 0.025 & 0.769 ± 0.025 & 0.709 ± 0.046 & 0.668 ± 0.054 \\
P-UMAP & 0.777 ± 0.027 & 0.762 ± 0.025 & 0.695 ± 0.030 & 0.676 ± 0.037 \\
SPCA* & 0.785 ± 0.024 & 0.777 ± 0.026 & 0.753 ± 0.026 & 0.749 ± 0.026 \\
PLS-DA* & 0.724 ± 0.022 & 0.714 ± 0.022 & 0.654 ± 0.023 & 0.634 ± 0.025 \\
CEBRA* & 0.806 ± 0.046 & 0.784 ± 0.049 & 0.736 ± 0.058 & 0.731 ± 0.065 \\
PCA & 0.755 ± 0.023 & 0.752 ± 0.022 & 0.741 ± 0.023 & 0.736 ± 0.024 \\

\\
\multicolumn{5}{c}{Aggregated Importances} \\
RF-AE* & \textbf{\textit{0.803 ± 0.040}} & \textit{0.818 ± 0.038} & \underline{\textbf{\textit{0.780 ± 0.055}}} & \underline{\textbf{\textit{0.784 ± 0.053}}} \\
SSNP* & 0.764 ± 0.060 & 0.770 ± 0.056 & 0.685 ± 0.100 & 0.695 ± 0.092 \\
P-SUMAP* & 0.757 ± 0.046 & 0.767 ± 0.044 & 0.647 ± 0.065 & 0.646 ± 0.063 \\
CE* & \textit{0.798 ± 0.062} & \textbf{\textit{0.819 ± 0.063}} & \textit{0.771 ± 0.068} & \textbf{\textit{0.772 ± 0.067}} \\
RF-PHATE* & \textit{0.798 ± 0.048} & \underline{\textbf{\textit{0.825 ± 0.046}}} & 0.753 ± 0.049 & 0.756 ± 0.050 \\
NCA* & \underline{\textbf{\textit{0.812 ± 0.045}}} & 0.804 ± 0.046 & \textbf{\textit{0.774 ± 0.048}} & \textit{0.762 ± 0.049} \\
PACMAP & 0.749 ± 0.046 & 0.758 ± 0.044 & 0.688 ± 0.044 & 0.690 ± 0.046 \\
P-TSNE & 0.744 ± 0.044 & 0.747 ± 0.044 & 0.684 ± 0.048 & 0.667 ± 0.051 \\
AE & 0.745 ± 0.043 & 0.750 ± 0.045 & 0.695 ± 0.061 & 0.655 ± 0.066 \\
P-UMAP & 0.760 ± 0.047 & 0.744 ± 0.047 & 0.674 ± 0.048 & 0.657 ± 0.054 \\
SPCA* & 0.770 ± 0.042 & 0.761 ± 0.047 & 0.742 ± 0.045 & 0.739 ± 0.046 \\
PLS-DA* & 0.717 ± 0.043 & 0.710 ± 0.044 & 0.664 ± 0.039 & 0.643 ± 0.040 \\
CEBRA* & 0.782 ± 0.063 & 0.778 ± 0.068 & 0.739 ± 0.073 & 0.733 ± 0.079 \\
PCA & 0.746 ± 0.045 & 0.743 ± 0.044 & 0.734 ± 0.042 & 0.729 ± 0.043 \\

\bottomrule
\end{tabular}
\end{sc}
\end{table}

\section{Extended visualizations and quantitative comparisons}\label{sec:oos_viz_supplemental}
We present OOS visualization plots and quantitative comparison (Table~\ref{tab:quantitative_per_dataset}) for all models on Sign MNIST (Fig.~\ref{fig:viz_sign_mnist_full}) and OrganC MNIST dataset (Fig.~\ref{fig:viz_organ_full}) to support our analysis in Section~\ref{subsec:oos_viz}. 

Table~\ref{tab:quantitative_per_dataset} shows the local ($s=\textit{QNX},\textit{Trust}$) and global ($s=\textit{Spear},\textit{Pearson}$) SIA scores, along with test $k$-NN accuracies for RF-AE and 13 baseline methods on the Sign MNIST and OrganC MNIST datasets. Our RF-AE method consistently ranks among the top three across all scores on both datasets.

Fig.~\ref{fig:viz_sign_mnist_full} presents visualizations of all models for the Sign MNIST (A–K) dataset. RF-AE effectively inherits the global structure of the RF-PHATE embeddings while providing greater detail within class clusters. In contrast, RF-PHATE tends to compress representations within each cluster, which are associated with individual classes. Although OOS embeddings are mostly assigned to their correct ground truth labels, the local arrangement of these samples on the sub-manifold is not easily visualized in RF-PHATE. RF-AE, however, expands the class clusters, revealing within-class patterns that are obscured in the RF-PHATE plot. For example, the top-right cluster in the RF-AE plot illustrates different ways to represent the letter ``C'', showing a logical transition between variations based on hand shadowing and orientation. Such nuanced differences are more challenging to detect in RF-PHATE, which compresses these representations into an overly restrictive branch structure. This limitation of RF-PHATE may stem from excessive reliance on the diffusion operator, which overemphasizes global smoothing. Since RF-GAP already captures local-to-global supervised neighborhoods effectively, the additional diffusion applied by RF-PHATE likely diminishes fine-grained local details. Thus, we have demonstrated that RF-AE offers a superior balance for visualizing the local-to-global supervised structure compared to the basic RF-PHATE kernel extension. 

P-TSNE is effective at identifying clusters of similar samples but often splits points from the same class into distinct, distant clusters. This appears to result from variations such as background shadowing, which obstruct the important part of the image. Thus, ``G'' and ``H'' instances are closer than expected due to similar shadowing. In contrast, RF-AE correctly assigns ``G'' and ``H'' instances to their own clusters while dissociating between same-class points with different shadowing, effectively reflecting within-class variations. This demonstrates that P-TSNE is overly sensitive to irrelevant factors, such as background differences, which are unrelated to the underlying labels.  Similarly, P-UMAP, P-SUMAP and PACMAP exhibit this sensitivity but produces sparser representations. Despite being a supervised method, P-SUMAP incorporates class labels in a way that artificially clusters same-class points, potentially oversimplifying their intrinsic relationships.
CEBRA yields a circular pattern that offers limited utility for qualitative interpretation. CE and SSNP over-separates the classes without connections between them, further hindering the visualization of inter-class relationships. NCA retains decent local and global relationships, but within-class variations and transitions between classes are visually less evident than in regularized RF-AE. Other methods produced noisy embeddings.



For the OrganC MNIST dataset, all models are visualized in Fig.~\ref{fig:viz_organ_full}. As analyzed in Section~\ref{subsec:oos_viz}, RF-AE achieves notable improvements over competing methods by enabling finer distinctions between organ types. This is particularly evident in its ability to differentiate the left and right kidneys—whereas other methods tend to merge these classes, RF-AE separates them while maintaining their proximity in the embedding space. This reflects anatomical similarity without losing class identity. 

RF-PHATE maintains global structure but merges similar classes, reducing fine-grained resolution. P-TSNE, P-UMAP, and P-SUMAP form overlapping clusters, leading to cluttered embeddings that hinder interpretation. NCA, SPCA, and PCA create scattered or overlapping layouts with weak separation of organ types, indicating limited class-specific representation. PACMAP and CE show broken structures, where organ clusters are split or linearly aligned without clear biological meaning. SSNP and CEBRA fail to produce informative embeddings, with SSNP collapsing the data and CEBRA displaying an artificial circular pattern. PLS-DA and AE partially separate some classes but mix others, reflecting suboptimal balance between class identity and spatial organization.

Overall, RF-AE preserves the structural integrity of the data while substantially enhancing class separability. These qualitative observations are consistent with the quantitative results in Table~\ref{tab:quantitative_per_dataset}, where RF-AE outperforms other methods in both $k$-NN accuracy and local-to-global SIA. 

\begin{table}[ht]
\caption{Local ($s=\textit{QNX},\textit{Trust}$) and global ($s=\textit{Spear},\textit{Pearson}$) SIA scores, along with test $k$-NN accuracies for our RF-AE method and 13 baselines. Scores are shown as mean $\pm$ std across 10 repetitions. Top three values in each column are highlighted with underlined bold italics (first), bold italics (second), and italics (third). Supervised methods are marked by an asterisk.}
\label{tab:quantitative_per_dataset}

\centering
\small
\begin{sc}
\begin{tabular}{lccccc}
\toprule
 & \multicolumn{2}{c}{Local SIA} & \multicolumn{2}{c}{Global SIA} & \\
 \cmidrule(r){2-5}
 & QNX & Trust & Spear & Pearson & \multicolumn{1}{c}{$k$-NN Acc} \\
\midrule

\multicolumn{6}{c}{Sign MNIST} \\
RF-AE* & \textbf{\textit{0.816 ± 0.007}} & \textit{0.851 ± 0.013} & \textbf{\textit{0.724 ± 0.076}} & \textbf{\textit{0.719 ± 0.084}} & \textbf{\textit{0.983 ± 0.003}} \\
SSNP* & 0.139 ± 0.401 & 0.249 ± 0.381 & 0.174 ± 0.391 & 0.414 ± 0.219 & 0.189 ± 0.258 \\
P-SUMAP* & 0.700 ± 0.010 & 0.618 ± 0.009 & 0.449 ± 0.079 & 0.401 ± 0.103 & 0.967 ± 0.004 \\
CE* & 0.620 ± 0.408 & 0.627 ± 0.418 & \textit{0.695 ± 0.135} & \textit{0.646 ± 0.184} & 0.464 ± 0.179 \\
RF-PHATE* & \underline{\textbf{\textit{0.821 ± 0.008}}} & \textbf{\textit{0.861 ± 0.011}} & 0.574 ± 0.103 & 0.450 ± 0.156 & \textit{0.975 ± 0.004} \\
NCA* & \textit{0.793 ± 0.013} & \underline{\textbf{\textit{0.873 ± 0.012}}} & 0.596 ± 0.088 & 0.523 ± 0.110 & \underline{\textbf{\textit{0.984 ± 0.002}}} \\
PACMAP & 0.718 ± 0.007 & 0.616 ± 0.008 & 0.402 ± 0.026 & 0.382 ± 0.029 & 0.930 ± 0.005 \\
P-TSNE & 0.689 ± 0.010 & 0.535 ± 0.021 & 0.304 ± 0.050 & 0.210 ± 0.084 & 0.806 ± 0.032 \\
AE & 0.668 ± 0.019 & 0.625 ± 0.046 & 0.403 ± 0.165 & 0.361 ± 0.181 & 0.524 ± 0.131 \\
P-UMAP & 0.665 ± 0.012 & 0.551 ± 0.011 & 0.304 ± 0.042 & 0.223 ± 0.064 & 0.787 ± 0.026 \\
SPCA* & 0.676 ± 0.005 & 0.598 ± 0.009 & 0.552 ± 0.011 & 0.519 ± 0.012 & 0.479 ± 0.009 \\
PLS-DA* & 0.740 ± 0.008 & 0.729 ± 0.009 & \underline{\textbf{\textit{0.737 ± 0.011}}} & \underline{\textbf{\textit{0.735 ± 0.012}}} & 0.357 ± 0.008 \\
CEBRA* & 0.742 ± 0.064 & 0.744 ± 0.132 & 0.586 ± 0.129 & 0.564 ± 0.149 & 0.430 ± 0.091 \\
PCA & 0.660 ± 0.011 & 0.588 ± 0.015 & 0.576 ± 0.013 & 0.589 ± 0.012 & 0.314 ± 0.006 \\
\\

\multicolumn{6}{c}{OrganC MNIST} \\
RF-AE* & \textbf{\textit{0.898 ± 0.008}} & \underline{\textbf{\textit{0.930 ± 0.009}}} & \textbf{\textit{0.898 ± 0.014}} & \underline{\textbf{\textit{0.900 ± 0.012}}} & \underline{\textbf{\textit{0.761 ± 0.010}}} \\
SSNP* & 0.871 ± 0.028 & 0.906 ± 0.019 & 0.773 ± 0.358 & 0.784 ± 0.096 & \textit{0.636 ± 0.154} \\
P-SUMAP* & 0.873 ± 0.006 & 0.898 ± 0.006 & 0.886 ± 0.006 & 0.875 ± 0.006 & 0.618 ± 0.018 \\
CE* & 0.870 ± 0.022 & 0.887 ± 0.024 & 0.854 ± 0.076 & 0.846 ± 0.073 & 0.570 ± 0.193 \\
RF-PHATE* & \textit{0.893 ± 0.006} & \textit{0.913 ± 0.005} & \underline{\textbf{\textit{0.899 ± 0.009}}} & \textbf{\textit{0.898 ± 0.013}} & \textbf{\textit{0.656 ± 0.009}} \\
NCA* & 0.892 ± 0.006 & 0.896 ± 0.005 & 0.870 ± 0.005 & 0.868 ± 0.005 & 0.524 ± 0.000 \\
PACMAP & 0.881 ± 0.007 & 0.902 ± 0.006 & \textit{0.893 ± 0.007} & \textit{0.893 ± 0.006} & 0.632 ± 0.009 \\
P-TSNE & 0.867 ± 0.006 & 0.892 ± 0.005 & 0.874 ± 0.005 & 0.871 ± 0.005 & 0.474 ± 0.003 \\
AE & 0.875 ± 0.006 & 0.899 ± 0.005 & 0.873 ± 0.011 & 0.834 ± 0.022 & 0.563 ± 0.014 \\
P-UMAP & 0.881 ± 0.006 & 0.898 ± 0.004 & 0.870 ± 0.005 & 0.868 ± 0.005 & 0.475 ± 0.005 \\
SPCA* & \underline{\textbf{\textit{0.916 ± 0.005}}} & \textbf{\textit{0.926 ± 0.005}} & \textit{0.895 ± 0.005} & 0.886 ± 0.005 & 0.429 ± 0.000 \\
PLS-DA* & 0.866 ± 0.006 & 0.869 ± 0.005 & 0.860 ± 0.005 & 0.859 ± 0.005 & 0.358 ± 0.000 \\
CEBRA* & 0.858 ± 0.033 & 0.881 ± 0.032 & 0.872 ± 0.030 & 0.862 ± 0.027 & 0.358 ± 0.034 \\
PCA & 0.861 ± 0.005 & 0.879 ± 0.005 & 0.865 ± 0.005 & 0.861 ± 0.005 & 0.414 ± 0.000 \\
\bottomrule
\end{tabular}
\end{sc}
\end{table}

\begin{figure}[t]
    \centering
    \includegraphics[width = \textwidth]{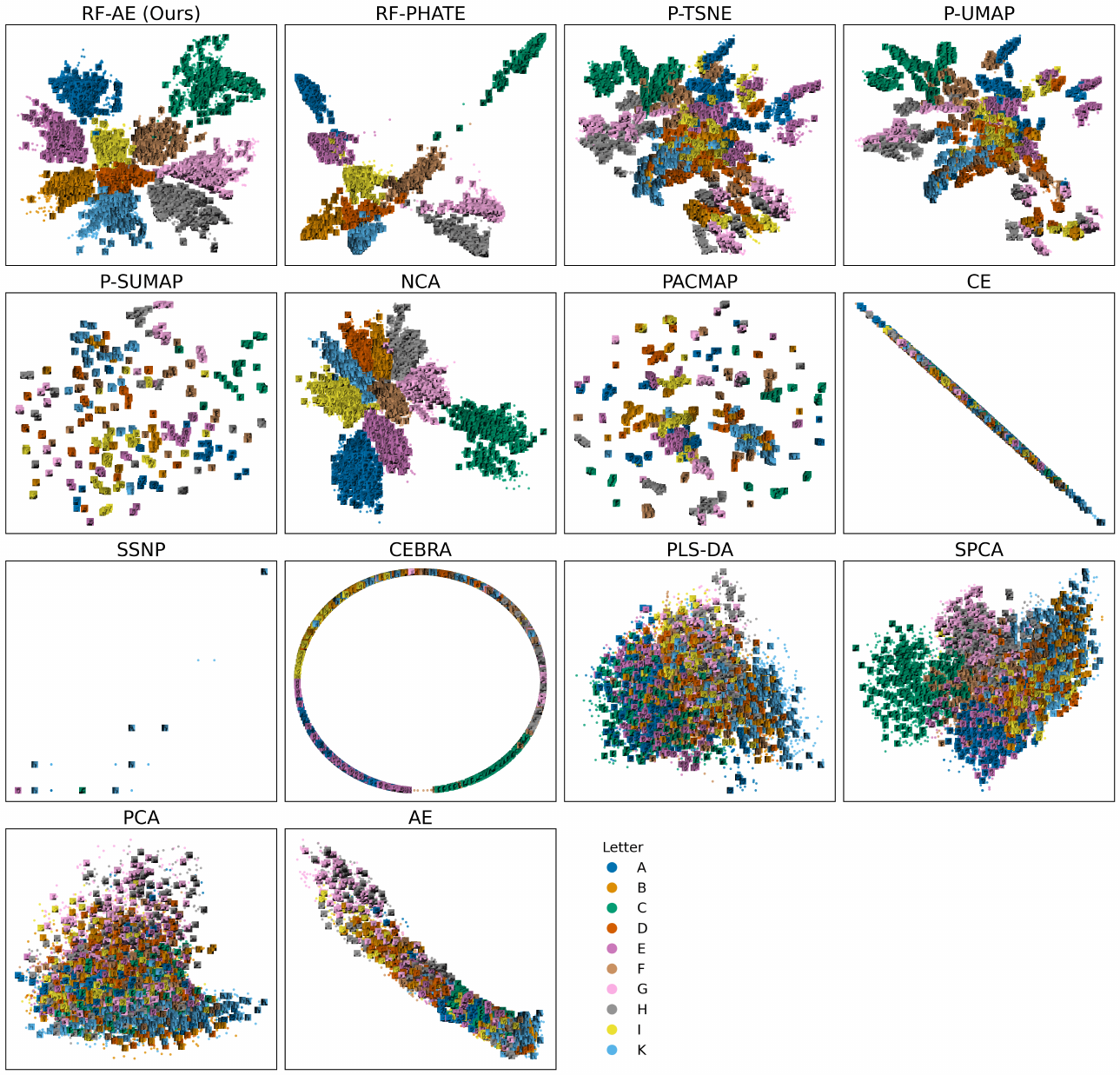}
    \caption{Visualization of the Sign MNIST (A--K) dataset (Table~\ref{tab:rfae_data}) using 14 dimensionality reduction methods. Training points are shown as color-coded circles based on their labels, while test points are displayed with their original images, tinted to match their labels. Refer to Section~\ref{subsec:oos_viz} and Appendix~\ref{sec:oos_viz_supplemental} for a full qualitative analysis.}
\label{fig:viz_sign_mnist_full}
\end{figure}


\begin{figure}[t]
    \centering
    \includegraphics[width = \textwidth]{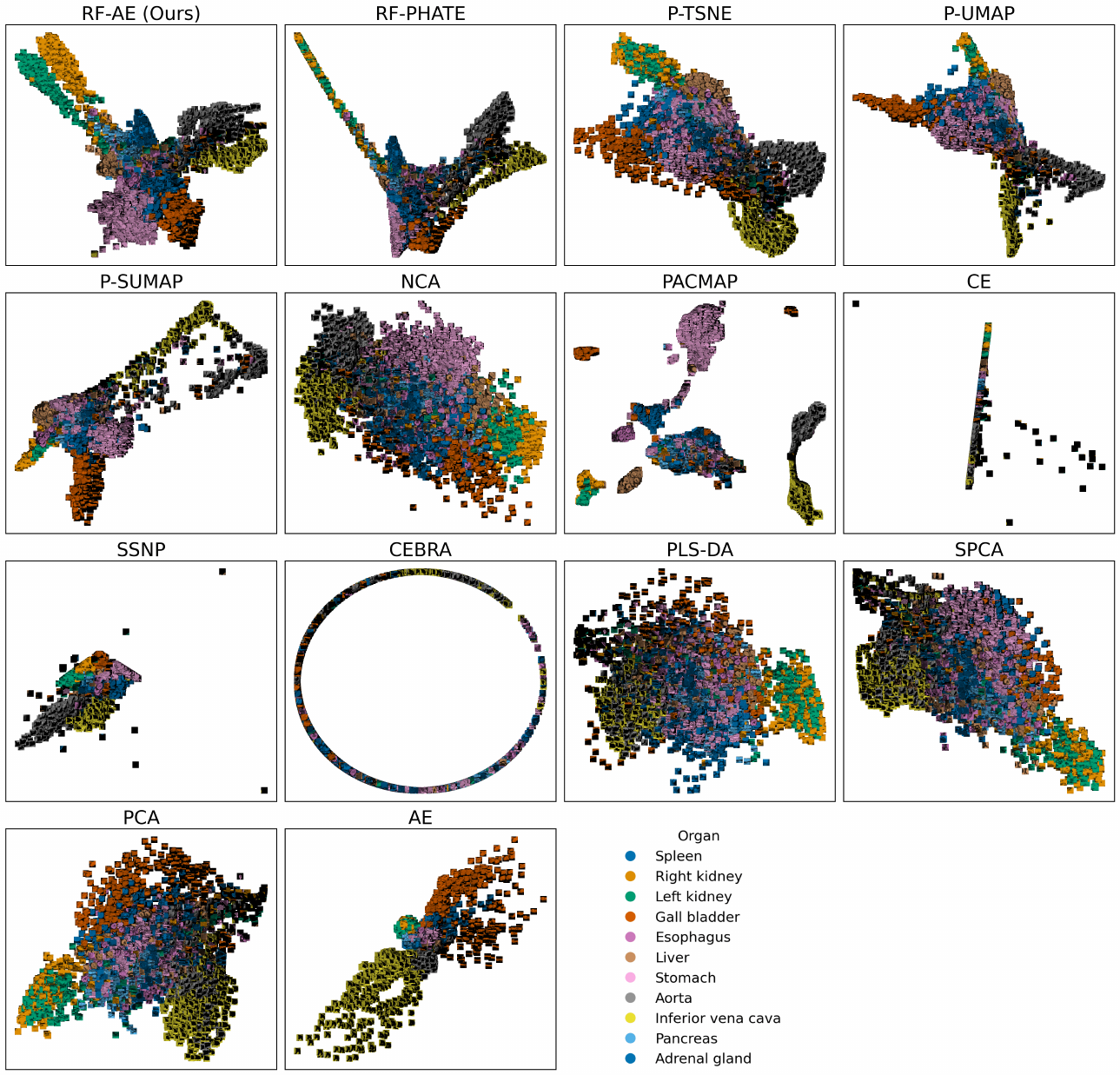}
    \caption{Visualization of the OrganC MNIST dataset (Table~\ref{tab:rfae_data}) using 14 dimensionality reduction methods. Test points are displayed with their original images, tinted to match their labels. Training points are omitted for clarity. Refer to Appendix~\ref{sec:oos_viz_supplemental} for a full qualitative analysis.}
    \label{fig:viz_organ_full}
\end{figure}

\section{Broader impacts}
\label{sec:impacts}
This paper advances guided data representation learning by integrating expert-derived data annotations and enabling out-of-sample extension, allowing generalization beyond the training set. The proposed method assists decision-makers with interpretable visualizations while improving scalability and applicability in semi-supervised tasks. In particular, RF-AE can support expert- or AI-based disease diagnosis by projecting incoming patient instances into a 2D space, where they can be contextualized relative to existing embeddings. This visualization enables practitioners to assess whether a given prediction aligns well with known structures or deviates from them, offering a valuable indicator of prediction reliability. Overall, the method has potential societal impacts in biomedical research, as well as data-driven insights in healthcare, finance, and multimedia.


\newpage
\clearpage

\section*{NeurIPS Paper Checklist}

\begin{enumerate}

\item {\bf Claims}
    \item[] Question: Do the main claims made in the abstract and introduction accurately reflect the paper's contributions and scope?
    \item[] Answer: \answerYes{} 
    \item[] Justification: Our abstract and introduction accurately reflect the main contributions and scope of the paper. These claims are supported in the main text and appendix through methodological descriptions and empirical results. 

\item {\bf Limitations}
    \item[] Question: Does the paper discuss the limitations of the work performed by the authors?
    \item[] Answer: \answerYes{} 
    \item[] Justification: We explicitly discuss the limitations of our method in the discussion section (Section~\ref{sec:discussion}). 

\item {\bf Theory assumptions and proofs}
    \item[] Question: For each theoretical result, does the paper provide the full set of assumptions and a complete (and correct) proof?
    \item[] Answer: \answerNA{} 
    \item[] Justification: While the paper includes mathematical formulations of the proposed method and evaluation metrics, it does not present formal theoretical results with assumptions or proofs.

    \item {\bf Experimental result reproducibility}
    \item[] Question: Does the paper fully disclose all the information needed to reproduce the main experimental results of the paper to the extent that it affects the main claims and/or conclusions of the paper (regardless of whether the code and data are provided or not)?
    \item[] Answer: \answerYes{} 
    \item[] Justification: The paper provides sufficient details to reproduce the main experimental results, including a description of the model architecture (Section~\ref{subsec:rfae_arch}), training procedures (Section~\ref{subsec:quant_comp}), dataset (Appendix~\ref{sec:data}) and evaluation metrics (Section~\ref{subsec:quantify_oos_embedding}). Appendix~\ref{sec:rfae_exp_setting} also contains the full experimental settings to support reproducibility. 

\item {\bf Open access to data and code}
    \item[] Question: Does the paper provide open access to the data and code, with sufficient instructions to faithfully reproduce the main experimental results, as described in supplemental material?
    \item[] Answer: \answerYes{} 
    \item[] Justification: We release the full source code, preprocessing scripts, and instructions to reproduce all main experimental results upon publication. An anonymized version of the code and data access links are included in a single zip file to support reproducibility. 

\item {\bf Experimental setting/details}
    \item[] Question: Does the paper specify all the training and test details (e.g., data splits, hyperparameters, how they were chosen, type of optimizer, etc.) necessary to understand the results?
    \item[] Answer: \answerYes{} 
    \item[] Justification: The paper specifies all relevant experimental details, including data splits, model architecture, training hyperparameters, optimization settings, and evaluation metrics. These details are provided in the main text and further expanded in Appendix Section~\ref{subsec:exp-compute-resource} to ensure full transparency and reproducibility.

\item {\bf Experiment statistical significance}
    \item[] Question: Does the paper report error bars suitably and correctly defined or other appropriate information about the statistical significance of the experiments?
    \item[] Answer: \answerYes{} 
    \item[] Justification: The paper reports error bars for key experimental results, calculated across 10 runs with different random seeds over 20 different datasets, as described in Section~\ref{subsec:quant_comp} and Appendix Section~\ref{sec:ablation_supplemental}. The error bars represent the standard deviation of the metrics and are clearly indicated in the relevant tables.

\item {\bf Experiments compute resources}
    \item[] Question: For each experiment, does the paper provide sufficient information on the computer resources (type of compute workers, memory, time of execution) needed to reproduce the experiments?
    \item[] Answer: \answerYes{} 
    \item[] Justification: The paper provides sufficient information on the compute resources required to reproduce the experiments, including hardware specifications and runtime estimates. These details are documented in Appendix~\ref{sec:rfae_exp_setting}.
    
\item {\bf Code of ethics}
    \item[] Question: Does the research conducted in the paper conform, in every respect, with the NeurIPS Code of Ethics \url{https://neurips.cc/public/EthicsGuidelines}?
    \item[] Answer: \answerYes{} 
    \item[] Justification: This research presented in the paper fully complies with the NeurIPS Code of Ethics. 

\item {\bf Broader impacts}
    \item[] Question: Does the paper discuss both potential positive societal impacts and negative societal impacts of the work performed?
    \item[] Answer: \answerYes{} 
    \item[] Justification: The paper discusses potential positive and negative societal impacts in Appendix Section~\ref{sec:impacts}
    
\item {\bf Safeguards}
    \item[] Question: Does the paper describe safeguards that have been put in place for responsible release of data or models that have a high risk for misuse (e.g., pretrained language models, image generators, or scraped datasets)?
    \item[] Answer: \answerNA{} 
    \item[] Justification: The paper does not involve models or datasets that pose a high risk of misuse or dual use. The proposed method is intended for scientific research and does not generate or process sensitive or potentially harmful content.

\item {\bf Licenses for existing assets}
    \item[] Question: Are the creators or original owners of assets (e.g., code, data, models), used in the paper, properly credited and are the license and terms of use explicitly mentioned and properly respected?
    \item[] Answer: \answerYes{} 
    \item[] Justification: All external datasets (Appendix~\ref{sec:data}) and models (Section~\ref{subsec:quant_comp} used in the paper are properly cited with references to their original sources. 

\item {\bf New assets}
    \item[] Question: Are new assets introduced in the paper well documented and is the documentation provided alongside the assets?
    \item[] Answer: \answerYes{} 
    \item[] Justification: The paper introduces new models, which are described in detail in Section~\ref{sec:methods} and appendix. We provide code on model architecture, training procedures, and usage instructions to ensure reproducibility. 

\item {\bf Crowdsourcing and research with human subjects}
    \item[] Question: For crowdsourcing experiments and research with human subjects, does the paper include the full text of instructions given to participants and screenshots, if applicable, as well as details about compensation (if any)? 
    \item[] Answer: \answerNA{} 
    \item[] Justification: The paper uses publicly available single-cell dataset that was collected and shared by third parties. We do not conduct any new research involving human subjects or crowdsourcing, and the data used is ethically sourced.

\item {\bf Institutional review board (IRB) approvals or equivalent for research with human subjects}
    \item[] Question: Does the paper describe potential risks incurred by study participants, whether such risks were disclosed to the subjects, and whether Institutional Review Board (IRB) approvals (or an equivalent approval/review based on the requirements of your country or institution) were obtained?
    \item[] Answer: \answerNA{} 
    \item[] Justification:  The paper does not involve any direct research with human subjects or participant interaction. All data used are publicly available and were collected under proper ethical oversight by the original data providers.

\item {\bf Declaration of LLM usage}
    \item[] Question: Does the paper describe the usage of LLMs if it is an important, original, or non-standard component of the core methods in this research? Note that if the LLM is used only for writing, editing, or formatting purposes and does not impact the core methodology, scientific rigorousness, or originality of the research, declaration is not required.
    \item[] Answer: \answerNA{} 
    \item[] Justification: No large language models (LLMs) were used in the development of the core methods presented in this research. 

\end{enumerate}

\end{document}